\documentclass[a4paper,12pt]{article}
\usepackage{latexsym}
\usepackage{amsmath,amsthm}
\usepackage{amssymb}
\usepackage{algorithm}
\usepackage{bm}
\usepackage{bbold}
\usepackage{graphicx}
\usepackage{xcolor}
\usepackage[english]{babel}
\usepackage{natbib}
\usepackage{mathtools}

\usepackage{fancybox}

\newtheorem{lemma}{Lemma}
\newtheorem{rem}{Remark}
\newtheorem{theorem}{Theorem}
\newtheorem{proposition}{Proposition}
\newtheorem{definition}{Definition}
\newcommand{\blind}{0}

\date{}
\begin{document}
%\maketitle
%%%%%%%%%%%%%%%%%%%%%%%%%%%%%%%%%%%%%%%
%\bibliographystyle{natbib}

\def\spacingset#1{\renewcommand{\baselinestretch}%
{#1}\small\normalsize} \spacingset{1}

%%%%%%%%%%%%%%%%%%%%%%%%%%%%%%%%%%%%%%%%%%%%%%%%%%%%%%%%%%%%%%%%%%%%%%%%%%%%%%

\if0\blind
{
  \title{\bf High-Dimensional Penalized Bernstein Support Vector Machines}
  \author{Rachid Kharoubi{
 }\hspace{.2cm}\\
    Department of Mathematics, Universit\'e du Qu\'ebec \`A Montr\'eal\\
    Abdallah Mkhadri\\
    Departement of Mathematics,Cadi Ayyad University ,Faculty of Sciences Semlalia\\
    and \\
    Karim Oualkacha \\
   Department of Mathematics, Universit\'e du Qu\'ebec \`A Montr\'eal\\
   }
  \maketitle
} \fi

\if1\blind
{
  \bigskip
  \bigskip
  \bigskip
  \begin{center}
    {\LARGE\bf High-Dimensional Penalized Bernstein Support Vector Machines}
\end{center}
  \medskip
} \fi

\bigskip
\begin{abstract}
\noindent The support vector machines (SVM) is a powerful classifier used for binary classification to improve the prediction accuracy. However, the non-differentiability of the SVM hinge loss function can lead to computational difficulties in high dimensional settings. To overcome this problem, we rely on Bernstein polynomial and propose a new smoothed version of the SVM hinge loss called the Bernstein support vector machine (BernSVM), which is suitable for the high dimension $p >> n$ regime. As the BernSVM objective loss function is of the class $C^2$, we propose two efficient algorithms for computing the solution of the penalized BernSVM. The first algorithm is based on coordinate descent with maximization-majorization (MM) principle and the second one is IRLS-type algorithm (iterative re-weighted least squares). Under standard assumptions, we derive a cone condition and a restricted strong convexity to establish an upper bound for the weighted Lasso BernSVM estimator. Using a local linear approximation, we extend the latter result to penalized BernSVM with non convex penalties SCAD and MCP. Our bound holds with high probability and achieves a rate of order $\sqrt{s\log(p)/n}$, where $s$ is the number of active features. Simulation studies are considered to illustrate the prediction accuracy of BernSVM to its competitors and also to compare the performance of the two algorithms in terms of computational timing and error estimation. The use of the proposed method is illustrated through analysis of three large-scale real data examples.

\end{abstract}
%\begin{keywords}
{\bf Keywords} : SVM, Classification, Bernstein polynomial, Variables Selection, Non asymptotic Error Bound.
%\end{keywords}
%However, the hinge loss is not smooth and the Hubert loss his smoothed version 
%is not differentiable enough to develop the asymptotic properties in high dimensions. 

%\noindent In this paper, we propose a new smoothed version of the hinge loss called Bernstein loss.
%It is constructed using the Bernstein polynomial approximation.

%\noindent We present an L2 upper bound for the sparse estimators of the Bernstein penalized  SVM  with 
%convex and non-convex penalty

\newpage
\spacingset{1.75} % DON'T change the spacing!
\section{Introduction}
The SVM, \citep{cortes1995support}, is a powerful classifier used for binary classification to perform the 
prediction accuracy. Its motivation comes from geometric considerations and it seeks a hyperplane that separates two classes of data points by the largest margins (i. e. minimal distances of observations to a hyperplane). The classification rule depends only on a subset of observations, called support vectors, which lie along the lines indicating the width of the margin. It consists of classifying a test observation based on which side of the maximal margin hyperplane it lies. However, SVM can be less efficient in high dimensional setting with a large number of variables with only few of them are relevant for classification. Nowadays, problems with sparse scale data are common in applications such as finance, document classification, image analysis and gene expression analysis. Many sparse regularized SVM approaches are proposed to control the sparsity of the solution and to achieve both variable selection and classification. To list a few, SVM with $\ell_1$ penalty (\citep{bradley1998feature}; \citep{zhu20031}), SVM with elastic net \citep{wang2006doubly}, SVM with the SCAD penalty and SVM with the combination of SCAD and $\ell_2$ penalty \citep{becker2011elastic}. 
	
\noindent The objective function of these regularized SVM approaches is not differentiable at 1, so standard optimization techniques cannot be directly applied. To overcome this problem, recent alternatives are developed based on the main idea of approximating the hinge loss function with a modified smooth function which is differentiable. Then, the use of maximization-minimization (MM) principle together with a coordinate descent algorithm can be employed efficiently to update each component of the vector parameter of the SVM model. This idea was recently implemented in the generalized coordinate descent algorithm (gCDA) by \citep{yang2013efficient} and the SVM cluster correlation-network by \citep{kharoubi2019cluster}. Both authors considered the Huber loss function as a smooth approximation of the hinge loss function. They used the fact that the Huber loss is differentiable with a Lipschitz first derivative to build an upper-bound quadratic surrogate function for the Huber coordinate-wise objective function. Then, they solved the corresponding problem by a generalized coordinate descent algorithm using the MM principle to ensure the descent property. Other smooth approximations of the hinge loss can be found in the $\ell_2$ loss linear SVM of \citep{chang2008impact} and  \citep{lee2001rsvm}. 
Another algorithm that is known to be computationally efficient and might be adapted to solve penalized SVM is the iterative re-weighted least squares (IRLS-type) algorithm, which solves the regularized logistic regression \citep{friedman2010regularization}. However, all the objective functions of the aforementioned methods are not twice differentiable, and thus, efficient IRLS-type algorithms can not be designed to solve such approximate sparse SVM problems.
	
\noindent On the other side, some works on theoretical guaranties of sparse SVM (or $\ell_1$ SVM) are recently proposed, and are essentially based on the assumption that the Hessian of the SVM theoretical (expected) hinge loss is well-defined and continuous in order to use Taylor expansion or the continuity definition. For example, \citep{koo2008bahadur} studied the theoretical SVM loss function to overcome the differentiability of the hinge loss. Then, some appropriate assumptions about the distribution of the data $(\bm{x}_i,y_i)$, where $\bm{x}_i\in\mathbb{R}^p$ and $y_i\in\{-1,1\}, i=1,...,n$, were considered to ensure the existence and the continuity of the Hessian with respect to the vector of coefficients of the SVM regression model. \citep{koo2008bahadur} established asymptotic properties of the coefficients in low-dimension ($p<n$). \citep{zhang2016variable} used the same assumptions to establish the variable selection consistency in high-dimensions when the true vector of coefficients is sparse. \citep{dedieu2019error} exploited the same assumptions on the theoretical hinge loss in order to develop an upper bound of the error estimation of sparse SVM in high-dimension. 

%\textcolor{brown}{On the other hand, \citep{zhang2017oracle} derived an upper bound of the estimation error of the weighted Lasso in high dimensional additive hazards model. They used martingale to write the objective function in the least squares type loss function. The Hessian matrix is free of the coefficients, the same as the gram matrix in the least square context. They supposed the restricted strong convexity of the Hessian matrix to derive some oracle inequalities.}
%\textcolor{blue}{On the other hand, \citep{zhang2017oracle}, derived an upper bound of the estimation error of the weighted Lasso in high dimensional additive hazards model. They transformed the objective function on a least square loss kind. They suppose that the gram matrix have the restricted strong convexity property \citep{negahban2009unified}}.

\noindent Motivated by the computation problem of the SVM hinge loss and the discontinuity of the second derivative of the hessian of Huber hinge loss, we propose the BernSVM loss function, which is based on the  Bernstein polynomial approximation. Its second derivative is well-defined, continuous and bounded, which simplifies the implementation of a generalized coordinate descent algorithm with strict descent property and the implementation of an IRLS-type algorithm to solve penalized SVM. By construction, the proposed BernSVM loss function has suitable properties, which allows its Hessian to satisfy a restricted strong convexity \citep{negahban2009unified}. This is essential for non asymptotic theoretical guaranties developed in this work. More precisely and contrary to the works cited earlier, we consider the empirical loss function to establish an upper bound of the $\ell_2$ norm of the estimation error of the BernSVM weighted Lasso estimator. Similar to \citep{peng2016error}, we achieve an estimation error rate of order $\sqrt{s\log(p)/n}$.
		
\noindent This paper is organized as follows. In Section \ref{s2}, we give more details about the construction of the BernSVM loss function, using the Bernstein polynomials. Some properties of this function are oulined. Then we describe, in the same section, the two algorithms to solve the solution path of the penalized BernSVM. The first algorithm is a generalized coordinate descent and the second one is an IRLS type algorithm. In Section \ref{s3}, we undertake the theoretical properties of the penalized BernSVM with weighetd lasso penalty.  In particular, we derive an upper bound of the $\ell_2$ norm of the error estimation. Then, we extend this result to penalized BernSVM with a non-convex penalty (SCAD or MPC), using local linear approximation (LLA) algorithm. 
The empirical performance of BernSVM based on simulation studies and application to real datasets is detailed in Section \ref{s4}. Finally, a discussion and some conclusions are given in Section \ref{s5}.
\section{ The penalized BernSVM   }
\label{s2}
In this section we consider the penalized BernSVM in high dimension for binary classification with convex and non-convex penalties. We first present the BernSVM loss function, which is a spline polynomial of degree fourth, as a smooth approximation to the hinge loss function. Then, we present two efficient algorithms for 
computing the solution path of a solution to the penalized BernSVM problem. The first one is based on coordinate descent scheme and the second one is an IRLS type algorithm.

\noindent Assume we observe a training data of $n$ pairs,  $\{(y_1,{\mathbf x}_1),\ldots,(y_n,{\mathbf x}_n)\}$, where ${\mathbf x}_i\in I\!\!R^p$ and $y_i\in\{-1,1\}, i=1,...,n,$ denotes class labels.
\subsection{The fourth degree approximation spline}

We consider the problem of smoothing the SVM loss function $v$ : $ t\mapsto (1-t)_{+}$. 
 The idea is to fix some $\delta>0$, and to construct the simplest polynomial spline
\begin{equation}
\label{loss}
B_{\delta}(t)= \left\{\begin{array}{ll}
v(t)\ , & \mathrm{i}\mathrm{f}\ |t-1|\ >\delta,\\
g_{\delta}(t)\ , & \mathrm{i}\mathrm{f}\ |t-1|\ \leq\delta,
\end{array}\right.
\end{equation}
with $g_{\delta}$ is a decreasing polynomial, convex, and such that $B_{\delta}$ is twice continuously differentiable.
We have in particular $B_{\delta}(t) \rightarrow v(t)$, as $\delta\rightarrow 0$, for all $t \in \mathbb{R}$. These conditions can be rewritten in terms of the polynomial $g_{\delta}$ alone as
\begin{itemize}
\item[C1.] $g_{\delta}(1-\delta)=\delta$, $g_{\delta}^{'}(1-\delta)=-1$, and $g_{\delta}^{''}(1-\delta)=0$;
\item[C2.] $g_{\delta}(1+\delta)=0$, $g_{\delta}^{'}(1+\delta)=0$, and $g_{\delta}^{''}(1+\delta)=0$;
\item[C3.] $g_{\delta}^{''}\geq 0.$
\end{itemize}
Consider the affine transformation $ t\mapsto(t-1+\delta)/2\delta$, which maps the interval $[1-\delta,\ 1+\delta]$ on $[0$, 1$]$, and let
$$
q_{\delta}(x)=g_{\delta}(2\delta x+1-\delta)\ ,\ x\in\ [0,\ 1].
$$
In terms of $q_{\delta}$, the three conditions above become
\begin{itemize}
\item[C1'.] $q_{\delta}(0)=\delta$,  $q_{\delta}^{'}(0)=-2\delta$, and $q_{\delta}^{''}(0)=0$;
\item[C2'.] $q_{\delta}(1)=0$, $q_{\delta}^{'}(1)=0$, and $q_{\delta}^{''}(1)=0$;
\item[C3'.] $q_{\delta}^{''}\geq 0.$
\end{itemize}
These conditions can be dealt with in a simple manner in terms of a Bernstein basis.

\subsubsection{The BernSVM loss function}
The members of the Bernstein basis of degree $m$, $m\geq 0$, are the polynomials
$$
b_{k,m}(x)=\ \left(\begin{array}{l}
m\\
k
\end{array}\right)x^{k}(1-x)^{m-k},\ x\in\ [0,\ 1],
$$
for $0\leq k\leq m$. The coefficients of a polynomial $P$ in the Bernstein basis of degree $m$, will be denoted by $\{c(k,\ m;P)\ :\ k=0,\ .\ .\ .\ ,\ m\}$, and so we have
$$
P(x)=\sum_{k=0}^{m}c(k,\ m;P)b_{k,m}(x)\ ,\ x\in\ [0,\ 1].
$$
Let $\triangle$ be the forward difference operator. Here, when applied to a function $h$ of two arguments: $(k,\ m) \mapsto h(k,\ m)$, it is understood that $\triangle$ operates on the first argument:
%\newpage
$$
\triangle h(k,\ m)=h(k+1,\ m)-h(k,\ m)\quad \mbox{and}
$$
$$
\triangle^{2}h(k,\ m)=h(k+2,\ m)-2h(k+1,\ m)+h(k,\ m), \quad \mbox{for all} \quad (k, m).
$$ 
A basic fact related to the Bernstein basis is the following for $ 0\leq k\leq m,$ we have
$$
b_{k,m}'=m(b_{k-1,m-1}-b_{k,m-1})=-m\triangle b_{k-1,m-1} \quad \mbox{for}\quad  m\geq 1 \quad \mbox{qnd} \quad 0\leq k\leq m,  
$$
with the convention $b_{j,m}=0$ for $j\not\in\{0,\ .\ .\ .\ ,\ m\}.$

A useful consequence shows how derivatives of polynomials represented in the Bernstein basis act on the coefficients. Indeed, we have
\begin{itemize}
\item[(i)] $c(k,\ m-1;P')=m\triangle c(k,\ m;P)$, $0\leq k\leq m-1$,  $m\geq 1$, so that
$$
P^{'}(x)=m\sum_{k=0}^{m-1}\triangle c(k,\ m;P)b_{k,m-1}(x),\ x\in\ [0,\ 1],
$$
\item[(ii)] $c(k,\ m-2;P) =m(m-1)\triangle^{2}c(k,\ m;P)$, $0\leq k\leq m-2, m\geq 2$, so that
$$
P^{''}(x)=m(m-1)\sum_{k=0}^{m-2}\triangle^{2}c(k,\ m;P)b_{k,m-2}(x),\ x\in\ [0,\ 1].
$$
\end{itemize}
%\subsection{ The fourth degree approximating spline}
Then, it is easy to see that it is hopeless to find a cubic spline $B_{\delta}$ satisfying the constraints. If such a spline exists, its second derivative will be a polynomial spline of degree at most one, which must equal zero at the endpoints $ 1-\delta$ and $ 1+\delta$. By continuity of the second derivative, it therefore equals $0$ everywhere. The first derivative is therefore a degree zero polynomial on the entire axis and must equal $-1$ by the first condition. This contradicts the second condition.
%{\color{green}(Cette partie en ROUGE me semble inutile dans la mesure tu la rappelles dans la preuve}. The following theorem outlines the polynomial solution.
%\noindent {\color{purple}Normalement, pour le lissage d'une fonction, on cherche un polynome de degré 3. Cependant, les conditions C1 et C2 sur les dérivées secondes ne le permettent pas. Alors, nous allons chercher un polynome de degré 4. A mon avis, il faut la laisser parce qu'il y a th1 qui la suive et le lecteur ne pose pas la question pourquoi vous cherchez un poly de degré 4 a la place d'une spline.}
\begin{theorem}
The polynomial
\begin{equation}
\label{pdelta}
g_{\delta}(t)=\frac{1}{8\delta^{3}}\{\frac{(1-t+\delta)^{4}}{2}-(1-t-\delta)(1-t+\delta)^{3}\},\ t\in\ [1-\delta,\ 1+\delta]
\end{equation}
is the only degree $m=4$ polynomial satisfying the three conditions C1, C2 and C3.
\end{theorem}
\noindent The proof is differed to Appendix {\bf{A}}. 

\noindent The following proposition summarizes the properties of the first and the second derivative of the BernSVM loss function $B_{\delta}(.)$ defined in equation (1).
\begin{proposition}
For all $t\in\mathbb{R}$, we have $|B_{\delta}^{'}(t)|\le 1$ and $0\le B_{\delta}^{''}(t)\le \frac{3}{4\delta}$.
\end{proposition}

\begin{proof}
 \noindent The loss function $B_{\delta}(.)$ is twice continuously differentiable. Its first derivative is given by  
\begin{equation}
\label{loss1}
B_{\delta}^{'}(t)= \left\{\begin{array}{ll}
v^{'}(t)\ , & \mathrm{i}\mathrm{f}\ |t-1|\ >\delta,\\
g_{\delta}^{'}(t)\ , & \mathrm{i}\mathrm{f}\ |t-1|\ \leq\delta,
\end{array}\right.
\end{equation}
where $g_{\delta}^{'}(t) = \frac{(1-t+\delta)^2(1-t-2\delta)}{4\delta^3}$ and $v^{'}(t) = - \mathbf{1}_{\{1 - t > \delta \}} $.
Its second derivative is given by
\begin{equation}
\label{loss2}
B_{\delta}^{''}(t)= \left\{\begin{array}{ll}
0 \ , & \mathrm{i}\mathrm{f}\ |t-1|\ >\delta,\\
g_{\delta}^{''}(t)\ , & \mathrm{i}\mathrm{f}\ |t-1|\ \leq\delta,
\end{array}\right.
\end{equation}
where, $g_{\delta}^{''}(t) = \frac{3}{4\delta^3} [\delta^2 - (1-t)^2]$.
 We have, $0 \leq (1-t)^2 \leq \delta^2$ then, $0 \leq \delta^2 - (1-t)^2 \leq \delta^2$.
Thus, $0 \leq g_{\delta}^{''}(t) \leq \frac{3}{4\delta}$ and $g_{\delta}^{'}(t)$ is an increasing function.
We have also $g_{\delta}^{'}(1-\delta) = -1$ and $g_{\delta}^{'}(1 + \delta) = 0 $, then, 
$-1 \leq p_{\delta}^{'}(t) \leq  0$.
Thus, $\forall t \in \mathbb{R}$ we have $|B_{\delta}^{'}(t)|\leq 1$. \quad $\blacksquare$
\end{proof}
\subsubsection{Graphical comparaison of the three loss functions}
In this section, we examine some properties of the BernSVM loss function through a graphical illustration. We also compare its behavior with the standard SVM loss function $v(t)$ and the hinge loss function  considered in  HHSVM \citep{yang2013efficient}, which is given by 
\begin{displaymath}
\phi_c(t) =
\left\{
	\begin{array}{ll}
		0  & \mbox{, } t \geq 1, \\
		\frac{(1-t)^2}{2\delta} & \mbox{, } 1 -\delta <t\leq 1, \\
		1-t-\frac{\delta}{2} & \mbox{, }  t \leq 1-\delta. \\
	\end{array}
\right.
\end{displaymath}

\begin{figure}[H]
\caption{(a) The Huber loss (red) with $\delta = 0.01$, the BernSVM loss 
(blue) with $\delta = 0.01$ and the SVM loss (black), (b) The Huber loss (red) with $\delta = 2$, the BernSVM loss 
(blue) with $\delta = 2$ and the SVM loss (black)}
\begin{center}
\includegraphics[scale=0.8]{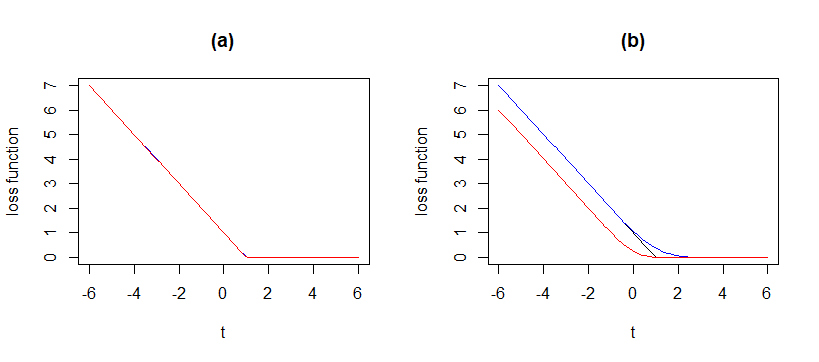} 
\end{center}
\end{figure}

\noindent Figure 1 shows this graphical illustration. When $\delta$ is small (i.e. close to zero), the three loss functions are almost identical (Panel (a)). 
When $\delta \geq 1$ (Panel (b)), both HHSVM and BernSVM loss functions have the same shape that the SVM loss. However, BernSVM approximates better the SVM loss function. Moreover, BernSVM loss function is twice differentiable everywhere. 

%\noindent The loss function $B_{\delta}(.)$ is a twice continuously differentiable function. Its first derivative is given by  
%\begin{equation}
%\label{loss1}
%B_{\delta}^{'}(t)= \left\{\begin{array}{ll}
%v^{'}(t)\ , & \mathrm{i}\mathrm{f}\ |t-1|\ >\delta,\\
%g_{\delta}^{'}(t)\ , & \mathrm{i}\mathrm{f}\ |t-1|\ \leq\delta,
%\end{array}\right.
%\end{equation}
%where $g_{\delta}^{'}(t) = \frac{(1-t+\delta)^2(1-t-2\delta)}{4\delta^3}$ and $v^{'}(t) = - \mathbf{1}_{\{1 - t > \delta \}} $.
%Its second derivative is given by
%\begin{equation}
%\label{loss2}
%B_{\delta}^{''}(t)= \left\{\begin{array}{ll}
%0 \ , & \mathrm{i}\mathrm{f}\ |t-1|\ >\delta,\\
%g_{\delta}^{''}(t)\ , & \mathrm{i}\mathrm{f}\ |t-1|\ \leq\delta,
%\end{array}\right.
%\end{equation}
%where, $g_{\delta}^{''}(t) = \frac{3}{4\delta^3} [\delta^2 - (1-t)^2]$.
% We have, $0 \leq (1-t)^2 \leq \delta^2$ then, $0 \leq \delta^2 - (1-t)^2 \leq \delta^2$.
%Thus, $0 \leq g_{\delta}^{''}(t) \leq \frac{3}{4\delta}$ and $g_{\delta}^{'}(t)$ is an increasing function.
%We have also $g_{\delta}^{'}(1-\delta) = -1$ and $g_{\delta}^{'}(1 + \delta) = 0 $, then, 
%$-1 \leq p_{\delta}^{'}(t) \leq  0$.
%Thus, $\forall t \in \mathbb{R}$ we have $|B_{\delta}^{'}(t)|\leq 1$.
\subsection{Algorithms for solving penalized BernSVM }
In this section, we propose the penalized BernSVM as an alternative to the penalized SVM and the penalized HHSVM for binary classification in high dimension settings. We are assuming $n$ pairs of training data $(\bm{x}_i,y_i)$,
for $i=1,...,n$, with $\bm{x}_i\in \mathbb{R}^p$ predictors and $y_i \in\{-1,1\}$.
We assume that the predictors are standardized:
$$\frac{1}{n}\sum_{i=1}^{n}x_{ij} = 0, \quad \frac{1}{n}\sum_{i=1}^{n}x_{ij}^2 = 1.$$

\noindent We propose to solve the penalized BernSVM problem given by 
\begin{equation}
\label{problem}
\underset{(\beta_0,\bm{\beta})}{\min} \frac{1}{n} \sum_{i=1}^{n}B_{\delta}(y_i(\beta_0 + \bm{x}_i^{\top}\bm{\beta})) + P_{\lambda_1,\lambda_2}(\bm{\beta}),
\end{equation}
where the objective function $B_{\delta}(.)$ is defined in (\ref{loss}) and 
$$
P_{\lambda_1,\lambda_2}(\bm{\beta}) = \sum_{j=1}^{p} P_{\lambda_1}(|\beta_j|) + \frac{\lambda_2}{2} ||\bm{\beta}||_2^2
$$ 
is a penalty function added for obtaining sparse coefficients' estimates. The penalty $P_{\lambda_1}(|\beta_j|)$ takes different forms: 
\begin{itemize}
 \item[$\bullet$] If $P_{\lambda_1}(|\beta_j|) = \lambda_1 |\beta_j|$,  $P_{\lambda_1,\lambda_2}(\bm{\beta})$ is the elastic net penalty (EN);
  \item[$\bullet$] If $P_{\lambda_1}(|\beta_j|) = \lambda_1w_j|\beta_j|$, where the weight $w_j > 0$ for all $j$, then $P_{\lambda_1,\lambda_2}(\bm{\beta})$ is the adaptive Lasso + L2 norm penalty (AEN);
  \item[$\bullet$] If one set 
  $$P_{\lambda_1}(|\beta_j|) =  \lambda_1 |\beta_j| \mathbb{1}_{|\beta_j|\leq \lambda_1} - 
\frac{|\beta_j|^2 - 2a\lambda_1|\beta_j| + \lambda_1^2}{2(a-1)}\mathbb{1}_{\lambda_1 < |\beta_j|\leq a\lambda_1} + \frac{(a+1)\lambda_1^2}{2}\mathbb{1}_{|\beta_j| > 
a \lambda_1},$$
with $a > 2$, then $P_{\lambda_1,\lambda_2}(\bm{\beta})$ is the  SCAD + L2 norm penalty (SCADEN);
 \item[$\bullet$] If one set
 $$
 P_{\lambda_1}(|\beta_j|)=\lambda_1 (|\beta_j| - \frac{|\beta_j|^2}{2\lambda_1 a})
			\mathbb{1}_{|\beta_j|< \lambda_1 a} + \frac{\lambda_1^2 a}{2}\mathbb{1}_{|\beta_j| \geq \lambda_1 a},
			$$ 
where $a > 1,$ then in this case, $P_{\lambda_1,\lambda_2}(\bm{\beta})$ is the  MCP + L2 norm penalty (MCPEN).
\end{itemize}
%\lambda_1 ||\bm{\beta}||_1
Before going through details of the proposed algorithms to solve BernSVM optimization problem, we would like to emphasize that, for small $\delta$, the minimizer of the penalized BernSVM in (\ref{problem}) is a good approximation to the minimizer of the penalized SVM. This is outlined in the following proposition.

%\begin{lemma}
\begin{proposition}
\noindent Let the penalized SVM loss function be defined as follows 
$$R(\bm{\beta},\beta_0) = \frac{1}{n} \sum_{i=1}^{n}v(y_i(\beta_0 + \bm{x}_i^{\top}\bm{\beta})) + P_{\lambda_1,\lambda_2}(\bm{\beta}),$$
and let the penalized BernSVM loss function be given as 
$$R(\bm{\beta},\beta_0|\delta) = \frac{1}{n} \sum_{i=1}^{n} B_{\delta}(y_i(\beta_0 + \bm{x}_i^{\top}\bm{\beta}))+ P_{\lambda_1,\lambda_2}(\bm{\beta}).$$
Then, we have 
$$\inf_{(\bm{\beta},\beta_0)} R(\bm{\beta},\beta_0) \leq \inf_{(\bm{\beta},\beta_0)} R(\bm{\beta},\beta_0|\delta) \leq \inf_{(\bm{\beta},\beta_0)} R(\bm{\beta},\beta_0)  + \delta.$$
%\end{lemma}
\end{proposition}
\begin{proof}
\noindent We have
$$ B_{\delta}(t) = v(t)\mathbf{1}_{(|t-1|>\delta)}(t) + g_{\delta}(t)\mathbf{1}_{(|t-1|\leq \delta)}(t),$$
where, $\mathbf{1}_{A}(t) = 1 $ if $t \in A$ and 0 otherwise. Then, one can write
$$B_{\delta}(t) - v(t) = \left\{\begin{array}{ll}
0\ , & \mathrm{i}\mathrm{f}\ |t-1|\ >\delta,\\
g_{\delta}(t) \geq 0\ , & \mathrm{i}\mathrm{f}\ |t-1|\ \leq\delta,
\end{array}\right.$$
which means that $v(t) \leq B_{\delta}(t)$. We have also that $g_{\delta}(t)$ is decreasing for any $t\in [1-\delta,1+\delta]$ because $g_{\delta}^{'}(t)\leq 0$. Thus, one has $g_{\delta}(t) \leq g_{\delta}(1-\delta) = \delta$, which  implies 
$$ v(t) + \delta - B_{\delta}(t) = \left\{\begin{array}{ll}
\delta \geq 0\ , & \mathrm{i}\mathrm{f}\ |t-1|\ >\delta\\
\delta - g_{\delta}(t) \geq 0\ , & \mathrm{i}\mathrm{f}\ |t-1|\ \leq\delta.
\end{array}\right.$$
We conclude that
$$v(t) \leq B_{\delta}(t) \leq v(t) + \delta.$$
This means that 
 $$R(\bm{\beta},\beta_0) \leq R(\bm{\beta},\beta_0|\delta) \leq R(\bm{\beta},\beta_0)  + \delta,$$
% Let $(\hat{\bm{\beta}}^{SD},\hat{\beta}_0^{SD})$ be the minimizer of the standard penalized SVM 
% and $(\hat{\bm{\beta}}^{\delta},\hat{\beta}_0^{\delta})$ is the minimizer of the smoothed penalized 
 %SVM.
 
\noindent which leads to 
$$\inf_{(\bm{\beta},\beta_0)} R(\bm{\beta},\beta_0) \leq \inf_{(\bm{\beta},\beta_0)} R(\bm{\beta},\beta_0|\delta) \leq \inf_{(\bm{\beta},\beta_0)} R(\bm{\beta},\beta_0)  + \delta. \quad \blacksquare
$$
\end{proof}
\noindent In the remaining of section \ref{s2}, we propose two competitor algorithms for solving (\ref{problem}) with convex penalties. The first one is a generalized coordinate descent algorithm based on MM principle, while the second 
one is an IRLS-type algorithm. Moreover, we briefly describe a third algorithm based 
on local linear approximation for solving (\ref{problem}) with non-convex penalties. 
 \subsubsection{The coordinate descent algorithm for penalized BernSVM}
In this section, we present a generalized coordinate descent (GCD) algorithm to solve penalized BernSVM, called BSVM-GCD, similar to the GCD approach proposed by \citep{yang2013efficient}. 
Because the second derivative of the BernSVM objective function is bounded, we approximate the coordinate wise objective function by a surrogate function and we use the coordinate descent to solve the optimization problem in (\ref{problem}). The proposed approach is described next.
 
\noindent Notice, first, that the coordinate objective function of problem (\ref{problem}) can be written as follows
 %Let also $F$ and $G$ be the densities of $\bm{X}$ given $\bm{y} = 1$ and $\bm{y} = -1$.
%\noindent Let $r_i = y_i (\tilde{\beta}_0 + \bm{x}_i^{\top}\tilde{\beta})$.
\begin{equation}
\label{3}
F(\beta_j|\beta_0,\tilde{\beta}_j) :=  \frac{1}{n} \sum_{i=1}^{n}B_{\delta}\{r_i + y_{i}x_{ij}(\beta_j - \tilde{\beta}_j)\} + P_{\lambda_1,\lambda_2}(\beta_j),
\end{equation}
where  $r_i = y_i (\tilde{\beta}_0 + \bm{x}_i^{\top}\tilde{\boldsymbol\beta})$, and $\tilde{\beta}_0$ and $\tilde{\boldsymbol\beta}$ are the current estimates of $\beta_0$ and $\boldsymbol\beta$, respectively.
% \noindent The following Lemma is used to approximate the objective function in (\ref{3})
%\begin{lemma}
%\noindent Let $f$ a loss function convex and twice differentiable. 
%The second order mean value  theorem of $f$ is given by 
%$$f(y) =  f(x) + f^{'}(x)(y - x) + \frac{f^{''}(t)}{2}(y - x)^2,$$
%for some $t$ between $x$ and $y$.
% If $f^{''}(.)$ is bounded by  $L>0$, then 
%$$f(y) \leq f(x) + f^{'}(x)(y - x) + \frac{L}{2}(y - x)^2.$$  
%\end{lemma}
As the loss function, $B_{\delta}(.)$, is convex, twice differentiable and has a second derivative 
bounded by $L = 3/4\delta$, one can rely on the mean value theorem and approximate 
the objective function in (\ref{3}) by a (surrogate) quadratic function, $Q_{\delta}(.)$, and then use the MM principle to solve the new problem for each $\beta_j$ holding $\tilde{\beta}_0$ and $\beta_k=\tilde{\beta}_j$, for $k\ne j$, fixed at the current iteration
\begin{equation}
\label{QFunct}
\underset{\beta_j)}{\min} Q_{\delta}(\beta_j|\tilde{\beta_0},\tilde{\beta}_j),
\end{equation}
where 
\begin{equation}
\label{obj}
Q_{\delta}(\beta_j|\beta_0,\tilde{\beta}_j) = \frac{\sum_{i=1}^{n}B_{\delta}(r_i)}{n} + 
\frac{\sum_{i=1}^{n}B_{\delta}^{'}(r_i)y_ix_{ij}}{n}(\beta_j - \tilde{\beta}_j) + \frac{L}{2}(\beta_j - \tilde{\beta}_j)^2 + P_{\lambda_1,\lambda_2}(\beta_j).
\end{equation}
\noindent  The next proposition gives the explicit solution
of (\ref{QFunct}) for the different forms of $P_{\lambda_1}(\beta_j)$.
\begin{proposition} 
Let $Q_{\delta}(\beta_j|\tilde{\beta}_0,\tilde{\beta}_j)$ be the surrogate loss function defined in (\ref{obj}). Let $P_{\lambda_1}(.)$ be $L_1$-norm or adaptive Lasso penalty. The closed form solution of the minimizer of (\ref{QFunct}) is given as follows	 
%\end{proposition}
\begin{equation}
\label{Elastic}
\hat{\beta}_j^{EN} = \frac{S(Z_j,\lambda_1)}{\omega},
\end{equation} 

\begin{equation}
\label{ada}
\hat{\beta}_j^{AEN} = \frac{S(Z_j,\lambda_1 w_j)}{\omega},
\end{equation} 
%\begin{equation}
%\label{SCADNET}
%\hat{\beta}_j^{Scadnet} = \frac{S(Z_j,\lambda_1)}{\omega} \textbf{I}_{\{|Z_j|\leq \lambda_1(\omega+1)\}} + \frac{(a-1)S(Z_j,\lambda_1\frac{a}{a -1})}{a \omega - \omega -1} \textbf{I}_{\{ \lambda_1 (\omega + 1)<|Z_j|< \lambda_1 \omega a\}} + \frac{Z_j}{\omega}\textbf{I}_{\{|Z_j| \geq a \lambda_1 \omega\}}
%\end{equation}
%\begin{equation}
%\label{MCPNET}
%\hat{\beta}_j^{Mcpnet} = \frac{a S(Z_j,\lambda_1)}{\omega a - 1} \textbf{I}_{\{|Z_j|\leq \lambda_1 \omega a\}}  + \frac{Z_j}{\omega}\textbf{I}_{\{|Z_j| > a \lambda_1 \omega\}}
%\end{equation}
where $S(a,b) = (|a| - b)_{+} sign(a), Z_j = -\frac{\sum_{i=1}^{n}B_{\delta}^{'}(r_i)y_ix_{ij}}{n} + L\tilde{\beta_j}$, and $\omega = \lambda_2 + L$ and $L=3/4\delta$.
\end{proposition}
%\noindent If $P_{\lambda_1,\lambda_2}(\beta_j)$ is the SCAD penalty, the solution is given by

%\begin{equation}
%\hat{\beta}_j = \frac{S(Z_j,\lambda_1)}{w} \textbf{I}_{\{|Z_j|\leq \lambda_1(w+1)\}} + \frac{(\gamma-1)S(Z_j,\lambda_1\frac{\gamma}{\gamma -1})}{\gamma w - w -1} \textbf{I}_{\{ \lambda_1 (w + 1)<|Z_j|< \lambda_1 w \gamma\}} + \frac{Z_j}{w}\textbf{I}_{\{|Z_j| \geq \gamma \lambda_1 w\}}
%\end{equation}

%\noindent If $P_{\lambda_1,\lambda_2}(\beta_j)$ is the MCP penalty, the solution is given by

%\begin{equation}
%\hat{\beta}_j = \frac{\gamma S(Z_j,\lambda_1)}{w\gamma - 1} \textbf{I}_{\{|Z_j|\leq \lambda_1 w \gamma\}}  + \frac{Z_j}{w}\textbf{I}_{\{|Z_j| > \gamma \lambda_1 w\}}
%\end{equation}
\noindent The proof is outlined in Appendix \textbf{A}. 

\noindent Algorithm \ref{al1} (BSVM-GCD), given next, gives details of steps for solving the objective function in (\ref{obj}) using both coordinate descent and MM principle. 
%\noindent However, in order to develop asymptotic properties of those estimators we use 
%the Bernstein  SVM  with Adapnet.
\begin{algorithm}[H]
	%\SetAlgoLined
	%\KwResult{The GSQR algorithm for GLasso penalty}
	\begin{enumerate}
		\item Initialize $\tilde{\beta}_0$ and $\tilde{\bm{\beta}}$;
		\item Iterate the following updates until convergence:
		\begin{enumerate}
			 
				\item For $j=1,\ldots,p$, update $\tilde{\beta}_j$ \\
				$*$ compute  $r_i = y_i(\tilde{\beta_0} + \textbf{x}^{\top}_{i.}\tilde{\bm{\beta}});$\\
				$*$ set 
				$$
				\tilde{\beta}_{j}^{\mathrm{new}} \longleftarrow \hat{\beta}_j,
				$$
				with $\hat{\beta}_j$ is given by (\ref{Elastic}) for the Elastic Net or (\ref{ada}) for the Adaptive Net;
				
				\item Update $\tilde{\beta}_0$ \\
				$*$ compute  $r_i = y_i(\tilde{\beta_0} + \textbf{x}^{\top}_{i.}\tilde{\bm{\beta}}),$\\
				$*$ set 
				$$
				\tilde{\beta}_0^{\mathrm{new}} \longleftarrow \tilde{\beta_0} - L^{-1}\frac{\sum_{i=1}^{n} B^{'}_{\delta}(r_i)y_i}{n}.
				$$            
			               
		\end{enumerate}
	\end{enumerate}
	\caption{The BSVM-GCD algorithm to solve the BernSVM loss function with elastic net and adaptive net penalties.}
 \label{al1}
\end{algorithm}
\noindent The upper bound of the second derivative of the BernSVM loss function can be reached in 
$t = 1$. This means that $F(t|\beta_0,\tilde{\beta}_j) \leq Q(t|\beta_0,\tilde{\beta}_j)$. To reach a strict inequality, one can relax the upper bound $L$ and use $\tilde{L} = (1 + \epsilon)L$, where $\epsilon = 1e^{-6}$. Hence, Algorithm \ref{al1} is implemented using $\tilde{L}$ in the place of $L$. This leads to the strict descent property of the BSVM-GCD algorithm, which is given in the next proposition.
%\begin{lemma}
\begin{proposition}
The strict descent property of the BSVM-GCD approach is obtained using the upper bound $\tilde{L}$ and given by 
$$F(\beta_j|\beta_0,\tilde{\beta}_j) = Q(\beta_j|\beta_0,\tilde{\beta}_j), \quad \text{if} \quad \beta_j = \tilde{\beta}_j,$$
$$F(\beta_j|\beta_0,\tilde{\beta}_j) < Q(\beta_j|\beta_0,\tilde{\beta}_j), \quad \text{if} \quad \beta_j \neq \tilde{\beta}_j.$$
\end{proposition}

\begin{proof}
Using Taylor expansion and the fact that the BernSVM loss function is twice differentiable, we have 
\begin{eqnarray*}
\frac{1}{n} \sum_{i=1}^{n}B_{\delta}\{r_i + y_{i}x_{ij}(\beta_j - \tilde{\beta}_j)\}
&=& \frac{1}{n} \sum_{i=1}^{n}B_{\delta}(r_i) + \frac{1}{n} \sum_{i=1}^{n}y_ix_{ij}B_{\delta}^{'}(r_i)(\beta_j - \tilde{\beta}_j) \\
&+& 
\frac{1}{2n} \sum_{i=1}^{n}y_i^2x_{ij}^2B_{\delta}^{''}(r_i)(\beta_j - \tilde{\beta}_j)^2 \\
&\leq& \frac{1}{n} \sum_{i=1}^{n}B_{\delta}(r_i) + \frac{1}{n} \sum_{i=1}^{n}y_ix_{ij}B_{\delta}^{'}(r_i)(\beta_j - \tilde{\beta}_j) + 
\frac{L}{2}(\beta_j - \tilde{\beta}_j)^2,
\end{eqnarray*}
The last inequality holds because we have $y_i^2 = 1$, $\sum_{i=1}^{n}x_{ij}^2 = n$, and $B_{\delta}^{''}(r_i)\leq L$ from Proposition 1. Hence, we have $F(\beta_j|\beta_0,\tilde{\beta}_j) = Q(\beta_j|\beta_0,\tilde{\beta}_j)$ if $\beta_j = \tilde{\beta}_j$, and we have $F(\beta_j|\beta_0,\tilde{\beta}_j) < Q(\beta_j|\beta_0,\tilde{\beta}_j)$, for $\beta_j \neq \tilde{\beta}_j$.
\end{proof}
\noindent This proposition shows that the objective function $F(.)$ decreases after each majorization  update, which means that for $\tilde{\beta}^{new}_j \neq \tilde{\beta}_j$ we have $$F(\beta^{new}_j|\beta_0,\tilde{\beta}_j) < F(\tilde{\beta}_j|\beta_0,\tilde{\beta}_j).$$
This result shows that the BSVM-GCD algorithm enjoys strict descent property.
%%%%%%%%%%%%%%%%%%%%%%%%%%%%%%%%%%%%%%%%%%%%%%%
%%% KO renoves all this paragraph
%\textcolor{purple}{We show next that the BSVM-GCD algorithm converge to the right solution if the update of the coordinate-coefficient not change. Furthermore,}
%\noindent{\color{red} the KKT conditions of the optimization problem in (\ref{obj}) is given by 
%$$\frac{\sum_{i=1}^{n}B_{\delta}^{'}(r_i)y_ix_{ij}}{n} - \tilde{L}\tilde{\beta_j} + \tilde{L}\beta_j + \lambda_1 sign(\beta_j) + \lambda_2 \beta_j  = 0,  \text{if} \quad \beta_j \neq 0,$$

%$$|\frac{\sum_{i=1}^{n}B_{\delta}^{'}(r_i)y_ix_{ij}}{n}| \leq \lambda_1,  \text{if} \quad \beta_j = 0.$$

%If $\beta^{new}_j = \tilde{\beta}_j$, the solution is stable, the KKT conditions become 
%$$\frac{\sum_{i=1}^{n}B_{\delta}^{'}(r_i)y_ix_{ij}}{n} + \lambda_1 sign(\beta_j) + \lambda_2 \beta_j  = 0, \text{if} \quad \beta_j \neq 0,$$

%$$|\frac{\sum_{i=1}^{n}B_{\delta}^{'}(r_i)y_ix_{ij}}{n}| \leq \lambda_1,  \text{if} \quad\beta_j = 0.$$

%This is exactly the KKT conditions of the original problem with the objective function $F(.)$.} \textcolor{brown}{This means that the proposed algorithm converges to the correct stationary point}({\color{green} QUEL EST L'INTERêT DE CETTE PARTIE en Rouge?})
%%%%%%%%%%%%%%%%%%%%%%%%%%%%%%%%%%%%%%%%%%%%%%%

\subsubsection{The IRLS algorithm for penalized BernSVM}
In this section, we present an IRLS-type algorithm combined with coordinate descent, termed BSVM-IRLS, to solve the penalized BernSVM in (\ref{problem}). In their approach, \citep{friedman2010regularization} used IRLS to solve the regularized logistic regression because the second derivative exist ({\tt glmnet} R package). On the other hand, \citep{yang2013efficient} used a GCD algorithm to solve the regularized HHSVM because the Huber loss does not
have the second derivative everywhere ({\tt gcdnet} R packge). Our BernSVM loss function, in contrast, is very smooth and thus one can rely on this advantage and solve (\ref{problem}) using IRLS trick combined with coordinate descent. The BSVM-IRLS approach is described below.
%\begin{equation}
%\label{1}
%\underset{(\beta_0,\bm{\beta})}{\min} L(\beta_0,\bm{\beta}),
%\end{equation}

\noindent Firstly, we consider the non-penalized BernSVM with loss function   
\begin{equation}
\label{2}
\ell{(\beta_0,\bm{\beta})} =  \frac{1}{n} \sum_{i=1}^{n}B_{\delta}(y_i(\beta_0 + \bm{x}_i^{\top}\bm{\beta})) = \frac{1}{n} \sum_{i=1}^{n}B_{\delta}(y_i\bm{X}_i^{\top}\bm{b}),
\end{equation}
where, $\bm{b} = (\beta_0,\bm{\beta}^{\top})^{\top}$ and $\bm{X} = (1,\bm{x}^{\top})^{\top}$. To solve 
$$
\frac{\partial{\ell}}{\partial \bm{b}}(\beta_0,\bm{\beta}) = \frac{1}{n} \sum_{i=1}^{n}y_i\bm{X}_iB^{'}_{\delta}(y_i\bm{X}_i^{\top}\bm{b}) = 0,
$$
one can use the Newton-Raphson algorithm because the loss function has a continuous second derivative defined by 
$$
\frac{\partial^{2}{\ell}}{\partial \bm{b}\partial \bm{b}^{\top}}(\beta_0,\bm{\beta}) = \frac{1}{n} \sum_{i=1}^{n}\bm{X}_iB^{''}_{\delta}(y_i\bm{X}_i^{\top}\bm{b})\bm{X}_i^{\top}.
$$ 
The update of the coefficients is given then by 
$$
\bm{b}^{new} = \tilde{\bm{b}} - (\frac{\partial^{2}{\ell}}{\partial \bm{b}\partial \bm{b}^{\top}}(\beta_0,\bm{\beta}))^{-1} \frac{\partial{\ell}}{\partial \bm{b}}(\beta_0,\bm{\beta})\tilde{\bm{b}}.
$$
Let $u_i = y_i.B^{'}_{\delta}(y_i\bm{X}_i^{\top}\tilde{\bm{b}})$, 
$\phi_{ii} = B^{''}_{\delta}(y_i\bm{X}_i^{\top}\tilde{\bm{b}})$ and $\tilde{\bm{\Phi}} = diag(\phi_{ii})$ for $i=1\ldots, n$.
Then, the update of the Newton-Raphson algorithm can be written as follows,
\begin{eqnarray*}
 \bm{b}^{new}& = &\tilde{\bm{b}} - (\bm{X}^{\top}\bm{\Phi}\bm{X})^{-1}\bm{X}^{\top}\bm{u} \\
 &= &(\bm{X}^{\top}\bm{\Phi}\bm{X})^{-1}\bm{X}^{\top}\bm{\Phi}[\bm{X}\tilde{\bm{b}} - \bm{\Phi}^{-1}\bm{u}]\\
 &= &(\bm{X}^{\top}\bm{\Phi}\bm{X})^{-1}\bm{X}^{\top}\bm{\Phi}\bm{z},
\end{eqnarray*}
where $\bm{z} = \bm{X}\tilde{\bm{b}} - \bm{\Phi}^{-1}\bm{u}$.
Thus, $\bm{b}^{new}$ is the update of a weighted least regression problem with response $\bm{z}$, with the loss function  
$$\ell{(\beta_0,\bm{\beta})} \approx \frac{1}{2n} \sum_{i=1}^{n} \phi_{ii}(z_i - \bm{X}_i^{\top}\bm{b})^2.$$
This problem can be solved by an IRLS algorithm in the case $p < n$. In high dimensional settings, the penalized BernSVM, defined in (\ref{problem}), can be solved by combining IRLS with a cyclic coordinate descent. Thus,
for the Adaptive Elastic Net penalty, the solution is given by 
\begin{equation}
\label{EN}
\hat{\beta}^{AEN}_j = \frac{S(\frac{\sum_{i=1}^{n}\phi_{ii}x_{ij} r_i}{n} + \frac{\sum_{i=1}^{n}\phi_{ii}x_{ij}^2 }{n}\tilde{\beta_j},w_j\lambda_1)}{\lambda_2 + \frac{\sum_{i=1}^{n}\phi_{ii}x_{ij}^2 }{n}} \quad \text{and} \quad
\hat{\beta}^{AEN}_0 = \frac{\sum_{i=1}^{n}\phi_{ii}r_i}{\sum_{i=1}^{n}\phi_{ii}},
\end{equation}
%and $$\hat{\beta}^{AEN}_0 = \frac{\sum_{i=1}^{n}\tilde{w}_{ii}r_i}{\sum_{i=1}^{n}\tilde{w}_{ii}},$$
where $r_i = z_i - \tilde{\beta}_0 - \bm{x}_i^{\top}\tilde{\bm{\beta}}$ and $(\tilde{\beta}_0,\tilde{\bm{\beta}})$ is the estimate of $({\beta}_0,\bm{\beta})$ obtained on the current iteration. Of note, to deal with the zero weights (i.e, $\phi_{ii}=0$ for some $i$), one can replace them by a constant $\xi = 0.001$ or use a modified Neweton-Raphson (NR) algorithm and replace all the weights by an upper bound of the second derivative. In our work, we adopted the modified NR algorithm and set $\phi_{ii} = L$ for $i=1,\ldots,n$, where $L = \frac{3}{4\delta}$ is the upper bound of the second derivative of the proposed loss function $B_{\delta}(.)$ given by (\ref{loss}).

%Thus, the problem to be solved using IRLS combined by coordinate descent is given by 

%$$\ell{(\beta_0,\bm{\beta})} \approx \frac{1}{2n} \sum_{i=1}^{n} L(z_i - \bm{X}_i^{\top}\bm{b}) +P_{\lambda_1,\lambda_2}(\bm{\beta}) .$$

\begin{algorithm}[H]
	%\SetAlgoLined
	%\KwResult{The GSQR algorithm for GLasso penalty}
	\begin{enumerate}
		\item Initialize $\tilde{\bm{b}} = (\tilde{\beta}_0,\tilde{\bm{\beta}}^{\top})^{\top}$ and let 
		$L =  \frac{3}{4\delta}$;
		\item Iterate the following updates until convergence:
		\begin{enumerate}
			    \item Calculate $\bm{u}$;
			    \item Calculate the working response $\bm{z} = \bm{X}\tilde{\bm{b}} - L^{-1}\bm{u}$;
			     \item Using cyclic coordinate descent to solve the penalized weighted Least square
		%	 \begin{enumerate}         
		%		\item For $j=1,\ldots,p$, update $\beta_j$ \\
		%		$*$ compute  $r_i = y_i(\tilde{\beta_0} + \textbf{x}^{\top}_{i.}\tilde{\bm{\beta}});$\\
		%		$*$ set 
		%		$$
		%		\tilde{\beta}_{j}^{\mathrm{new}} \longleftarrow \hat{\beta}_j^{*},
		%		$$
		%		with $\hat{\beta}_j^{*}$ is given by (\ref{EN});
			
		%		\item Update $\beta_0$ \\
		%		$*$ compute  $r_i = y_i(\tilde{\beta_0} + \textbf{x}^{\top}_{i.}\tilde{\bm{\beta}}),$\\
		%		$*$ set 
		%		$$
		%		\tilde{\beta}_0^{\mathrm{new}} \longleftarrow \tilde{\beta_0} - L^{-1}\frac{\sum_{i=1}^{n} S^{'}_{\delta}(r_i)y_i}{n}.
		%		$$            
		%	\end{enumerate}  
		  $$ \hat{\bm{b}} = (\hat{\beta}_0,\hat{\bm{\beta}})^{\top} = \underset{(\beta_0,\bm{\beta})}{arg\min} \ell{(\beta_0,\bm{\beta})},$$
		  where,% $\hat{\beta}_0 = \hat{\beta}^{AEN}_0$ and $\hat{\beta}_j = \hat{\beta}^{AEN}_j$
		  $(\hat{\beta}_0,\hat{\beta}_j), j=1,\ldots,p$, are given by (\ref{EN});
			   \item $\tilde{\bm{b}} = \hat{\bm{b}}$.             
		\end{enumerate}
	\end{enumerate}
	\caption{The BSVM-IRLS to solve the  BernSVM with elastic net.}
\end{algorithm}
\subsubsection{Local linear approximation algorithm for BernSVM with non-convex penalties}
We consider in this section the penalized BernSVM with non-convex penalty SCAD or MCP and we propose to solve the resulting optimization problem using the local linear approximation (LLA) of \citep{zou2008one}.

\noindent Without loss of generality, we assume in this section that $\lambda_2 = 0$. The LLA approach is an appropriate approximation of SCAD or MCP penalty. It is based on the first order Taylor expansion of the MCP or SCAD penalty functions around $|\tilde{\beta}_j|$, and adapts both penalties to be solved as an iterative adaptive Lasso penalty; the weights are updated/estimated at each iteration. In our context, this means that the penalized BernSVM with SCAD or MCP penalty can be solved iteratively as follows 
\begin{equation}
\label{weighted1}
(\tilde{\beta}_0,\tilde{\bm{\beta}})^{new} = \arg\min_{(\beta_0,\bm{\beta})} \frac{1}{n} \sum_{i=1}^{n}B_{\delta}(y_i(\beta_0+\bm{x}_i^{\top}\bm{\beta})) + \lambda_1 ||\hat{\bm{W}}\bm{\beta}||_1,
\end{equation}
where $\hat{\bm{W}}= diag\{\hat{w}_j\}$, with $\hat{w}_j = P^{'}_{\lambda_1}(\tilde{\beta_j})/\lambda_1, j=1,\ldots,p,$ are the current weights, and $P^{'}_{\lambda_1}(.)$ is the first derivative of the SCAD or MCP penalty given respectively by 
$$
P^{'}_{\lambda_1}(t) = \lambda_1 \mathbb{1}_{\{t \leq \lambda_1\}} + \frac{(a\lambda_1 - t)_+}{a - 1}\mathbb{1}_{\{t > \lambda_1\}},
$$
and 
$$
P^{'}_{\lambda_1}(t) = (\lambda_1 - \frac{t}{a})_+.
$$
To sum up, penalized BernSVM with LLA penalty is an iterative algorithm, whcih solves at each iteration BSVM-IRLS (or BSVM-GCD). The steps of the LLA algorithm for penalized BernSVM are described in Algorithm \ref{al3} below.

\begin{algorithm}[H]
	%\SetAlgoLined
	%\KwResult{The GSQR algorithm for GLasso penalty}
	\begin{enumerate}
		\item Initialize $\tilde{\beta}_0$ and $\tilde{\bm{\beta}}$ and compute $\hat{w}_j^{0} = P^{'}_{\lambda_1}(|\tilde{\beta_j}|)/\lambda_1$ for $j=1,\ldots,p$;
		\item Iterate the following updates until convergence:
		\begin{enumerate}
			
		%	\item For $j=1,\ldots,p$, update $\beta_j$ and $w_j$ \\
		%	$*$ compute  $r_i = y_i(\tilde{\beta_0} + \textbf{x}^{\top}_{i.}\tilde{\bm{\beta}});$\\
		%	$*$ set 
		%	$$
		%	\tilde{\beta}_{j}^{\mathrm{new}} \longleftarrow \hat{\beta}_j^{*},
		%	$$
		%	with $\hat{\beta}_j^{*}$ is given by (\ref{ada});
		 \item 	For $j=0,1,\ldots,p$, update $\beta_0$ and $\beta_j$ by solving the problem in (\ref{weighted1}) using either BSVM-IRLS or BSVM-GCD algorithms;
		 \item Set $\hat{w}_j^{\mathrm{new}} = P^{'}_{\lambda_1}(\tilde{\beta}_{j}^{\mathrm{new}})/\lambda_1$.
		%	\item Update $\beta_0$ \\
		%	$*$ compute  $r_i = y_i(\tilde{\beta_0} + \textbf{x}^{\top}_{i.}\tilde{\bm{\beta}}),$\\
		%	$*$ set 
		%	$$
		%	\tilde{\beta}_0^{\mathrm{new}} \longleftarrow \tilde{\beta_0} - \frac{4\delta}{3}\frac{\sum_{i=1}^{n} S^{'}_{\delta}(r_i)y_i}{n}.
	%		$$            
			
		\end{enumerate}
	\end{enumerate}
	\caption{The local linear approximation algorithm to solve penalized BernSVM with SCAD or MCP penalty.}
 \label{al3}
\end{algorithm}
\section{Theoretical guaranties for penalized BernSVM estimator}
\label{s3}
In the next section, we develop an upper bound of the $\ell_2$ norm of the estimation error $||\hat{\bm{\beta}} - \bm{\beta}^{\star}||_2$ for penalized BernSVM with  weighted lasso ($\lambda_2 = 0$) and SCAD or MCP penalties. Here, 
$\hat{\bm{\beta}}$ is the minimizer of the corresponding penalized BernSVM and 
$\bm{\beta}^{\star}$ is the minimizer of the population version of (\ref{problem}) without the penalty term; i.e., $\bm{\beta}^{\star}$ is the minimizer of $\mathbb{E}[B_{\delta}(y\bm{x}^\top\bm{\beta})]$. For seek of simplicity, the intercept is omitted for the theoretical development.
 
\subsection{An upper bound of the estimation error for weighted lasso penalty}
\noindent In this section, we derive an upper bound of the estimation error $||\hat{\bm{\beta}} - \bm{\beta}^{\star}||_2$ for penalized BernSVM with weighted lasso penalty using 
the properties of the BernSVM loss function. 

\noindent The weighted penalized BernSVM optimization problem is defined as follows
\begin{equation}
\label{weighted}
\arg\min_{\bm{\beta}\in \mathbb{R}^p}L(\bm{\beta}): = \frac{1}{n} \sum_{i=1}^{n}B_{\delta}(y_i\bm{x}_i^{\top}\bm{\beta}) + \lambda_1 ||\bm{W}\bm{\beta}||_1,
\end{equation}
where $ \bm{W} = diag(\bm{w})$ is a diagonal matrix of known weights, with $w_j \geq 0$, for $j=1\ldots p$.
% $$
% \hat{\bm{\beta}} = argmin_{\bm{\beta}\in \mathbb{R}^{p}}L(\bm{\beta}).
% $$ 
The (unweighted) Lasso is a special case of (\ref{weighted}), with $w_j = 1, j\in \{1,\ldots,p\}$. \\
Let $\hat{\bm{\beta}}$ be a minimizer of (\ref{weighted}), and assume that the population-level minimizer $\bm{\beta}^{\star}$ defined above is sparse and unique. Define $S\subset \{1,2,...,p\}$ the set of non-zero elements of $\bm{\beta}^{\star}$, with $|S| = s$, and $S^{c}$ is the complement of $S$.\\
In order to obtain an upper bound of the estimation error, the BernSVM loss function needs to satisfy the strongly convex assumption. This means that the minimum of the eigenvalues of its corresponding Hessian matrix, defined as 
$\bm{H} = \bm{X}^{\top}\bm{D}\bm{X}/n$, is lower bounded 
by a constant $\kappa > 0$; the matrix $\bm{D} = diag\{B_{\delta}^{"}(y_i\bm{x}_i^{\top}\bm{\beta}^{\star})\}$, $i=1,\ldots,n$.
In high-dimensional settings, with $p > n$, $\bm{H}$ has rank lower or equal to $n$. Therefore, we can not 
reach the strong convexity assumption of the Hessian matrix. To overcome this problem, one can look for a restricted strong convexity assumption to be verified on a restricted set $\mathcal{C}\subset \mathbb{R}^p$ \citep{negahban2009unified}, which is defined next.

%\noindent Let $\bm{W} = diag(w_j)$  with $w_j \geq 0$ are unknown weights.

%\noindent Like in Huang and Zhang (2012), we define 
%$$\Omega_0 = \{\hat{w}_j \leq w_j \quad \forall j \in \bm{S}\} \cup \{ w_j \leq \hat{w}_j\quad \forall j \in \bm{S}^{c}\}.$$
\begin{definition}[Restricted Strong Convexity (RSC)]
A design matrix $\bm{X}$ satisfies the RSC condition on  $\mathcal{C}$ with $\kappa > 0$ if
$$\frac{||\bm{X}\bm{h}||_2^2}{n} \geq \kappa ||\bm{h}||_2^2, \quad \forall \bm{h} \in \mathcal{C}.$$
\end{definition}

%Firstly, we give a result about the sub-Gaussian random variables
%\begin{lemma}
%$z_1,z_2,...,z_p$ are $p$ real random variables with $z_j \sim subG(\sigma^2)$ not necessarily independent, then for all $t>0$ $$\mathbb{P}(max_{j=1,...,p}|z_j|>t)\leq 2p\exp(-t^2/2\sigma^2).$$
%\end{lemma}
%We have the loss function $B_{\delta}(.)$ is Lipschitz with constant $L = \frac{3}{4\delta}$, then 
%our assumption 1 is the same as assumption  3 in Dedieu (2019).
\noindent We also state the following assumptions:
\begin{itemize}
\item[\textbf{A1:}]
%$$\sum_{i=1}^{n} - B_{\delta}^{'}(y_i\bm{x}_{i}^{\top}\bm{\beta}^{\star})y_ix_{ij} \sim subG(nL^2M^2) \quad \forall j.$$
Let $$\mathcal{C}(S) = \{\bm{h}\in \mathbb{R}^{p}:||\bm{W}_{\bm{S}^{c}}\bm{h}_{\bm{S}^{c}}||_1 \leq  \gamma ||\bm{h}_S||_1 \},$$ where $\gamma > 0$ and $\bm{h}_{\mathcal{A}}=(h_j:j\in \mathcal{A})$, with $\mathcal{A}\subset \{1,\ldots,p\}$, is a sub-vector of $\bm{h}$.
%$$\forall i \quad  \in \{1,...,N\}\quad  ||\bm{x}_i||_{\infty} \leq 1.$$\
\noindent We assume that the design matrix, for all $\bm{h} \in \mathcal{C}$, satisfies the condition 
$$\frac{||\bm{X}\bm{h}||_2^2}{n} \geq \kappa_1 ||\bm{h}||_2^2 - \kappa_2 \frac{\log(p)}{n}||\bm{h}||^2_1,$$ with probability $ 1- \exp(-c_0n)$, for some $c_0>0$, 
where $\bm{x}_i, i= 1,\ldots, n,$ are i.i.d  Gaussian or Sub-Gaussian.
\end{itemize}

\begin{itemize}
\item[\textbf{A2:}]
There exist a ball, $\bm{B}(\bm{x}_0,r_0)$, centered at a point $\bm{x}_0$, with radius $r_0$, such as 
$B^{''}_{\delta}\circ f(\bm{x}_0)  > 0$, with $f(\bm{x}_0) = y \bm{x}^{\top}_0 \bm{\beta}^{\star}$, and for any $
\bm{x} \in \bm{B}(\bm{x}_0,r_0)$, we have $B^{''}_{\delta}\circ f(\bm{x}) > \kappa_3$, for some constant $ \kappa_3> 0$.
%\noindent For some $i$, say $i_0$, we assume that the true parameter $\bm{\beta}^{\star}$ such as $t^{\star}_{i_0} = y_{i_0}\bm{x}_{i_0}^{\top}\bm{\beta}^{\star}$  satisfies $S^{"}_{\delta}(t^{\star}_{i_0}) > 0$.
% \in \mathcal{B}(R) = \{\bm{\beta}: ||\bm{\beta}||_1 \leq R\}$. 

% \noindent The continuity of $S^{"}_{\delta}(.)$ implies the existence of a Ball $\bm{B}(t^{\star}_{i_0},r^{\star})$ centred at $t^{\star}_{i_0}$ with radius $r^{\star}> 0$ such as 
% $\forall t \in \bm{B}(t^{\star},r^{\star})$ we have $S^{"}_{\delta}(t) > A > 0$.
\end{itemize}

\begin{itemize}
\item[\textbf{A3:}] 
 The columns of the design matrix $\bm{X}$ are bounded:
 $$\forall j\in\{1,...,p\} \quad ||\bm{x}_{j}||_2 \leq M.$$ 
\item[ \textbf{A4:}]
 The densities of $\bm{X}$ given $\bm{y}=1$ and $\bm{y}=-1$ are continuous and have at least the first moment.
\end{itemize}
\noindent \citep{raskutti2010restricted}  show that Assumption \textbf{A1} is satisfied with high probability for 
Gaussian random matrices.  \citep{rudelson2012reconstruction}  extend this result for Sub-Gaussian random matrices.
The existence of $\bm{x}_0$ in Assumption \textbf{A2} is given by the behavior of the second derivative $B^{''}_{\delta}(.)$ (see Proposition 1). In fact we have $B^{''}_{\delta}(1) = \frac{3}{4\delta}> 0$, then if we let $f(\bm{x}_0) = 1$,
we can choose $\bm{x}_0 \in f^{-1}(\{1\}) \subset \mathbb{R}^{p}$. Hence by the continuity of $B^{''}_{\delta}\circ f(.)$ we can choose $\kappa_3 = \frac{3}{4\delta}-\epsilon>0$ for some small $\epsilon>0$.
Note that this assumption is similar to Assumptions \textbf{A2} and \textbf{A4} in \citep{koo2008bahadur}, to ensure the positive-definiteness of the Hessian around $\bm{\beta}^{\star}$.
Assumption \textbf{A3} is needed to ensure that the first derivative of the loss function  is a bounded  random variable around $\bm{\beta}^{\star}$. Moreover, Assumption \textbf{A4} is needed to ensure the existence 
 of the first derivative of the population version of the loss function. 
 Assumption \textbf{A4} is similar to \textbf{A1} in \citep{koo2008bahadur}. However, in Assumption \textbf{A2} of \citep{koo2008bahadur}, one needs at least two moments to ensure the continuity  of the Hessian around $\bm{\beta}^{\star}$. 
 
\noindent The next lemma establishes that the hessian matrix $\bm{H}$ satisfies the RSC condition with high probability.
\begin{lemma}
\label{lem1}
%If $\delta = a + R$ for some $1 < a < R$ and $A = \frac{3}{4\delta^3}(a^2 - 1)$. 
Under assumptions $\bm{A} 1$ and $\bm{A} 2$, if $$n > \frac{2\kappa_2 s \log(p)(1 + \gamma ||\bm{W}_{S^{c}}^{-1}||_{\infty} )^2}{\kappa_1},$$ 
then the Hessian matrix $\bm{H}$ satisfies the RSC condition with 
$\kappa = \frac{\kappa_3\kappa_1}{2}$ on $\mathcal{C}$ 
%\cap \bm{B}(\bm{x}_0,r_0)$ 
with probability $ 1- \exp(-c_0n)$, for some universal constant $c_0>0$.
\end{lemma}
\noindent The proof of Lemma \ref{lem1} is given in Appendix {\bf B}.

%\noindent Let $\bm{A}_{\delta} = \{i: 1-\delta < y_i(\beta^{\star}_0 + \bm{x}_i^{\top}\bm{\beta}^{\star}) < 1 + \delta\}$.

%\noindent For $i \notin \bm{A}_{\delta}$ we have $B_{\delta}^{"}(y_i(\beta^{\star}_0 + \bm{x}_i^{\top}\bm{\beta}^{\star})) = 0$.

%\noindent Then, $$\bm{H} = \frac{\bm{X}_{\bm{A}_{\delta}}^{\top}\bm{D}_{\bm{A}_{\delta}}\bm{X}_{\bm{A}_{\delta}}}{n},$$
%where, $\bm{X}_{\bm{A}_{\delta}}$ is the sub-matrix of observations in $\bm{A}_{\delta}$ and $\bm{D}_{\bm{A}_{\delta}}$ is the diagonal matrix of the second derivative all positive.

%\noindent Let $s_{\min} = \min \{B_{\delta}^{"}(y_i(\beta^{\star}_0 + \bm{x}_i^{\top}\bm{\beta}^{\star})): i \in \bm{A}_{\delta}  \} > 0$.
%\noindent Thus, $$H \succeq \bm{H}_{\bm{A}_{\delta}} = \frac{s_{\min}{\bm{X}_{\bm{A}_{\delta}}^{\top}\bm{X}_{\bm{A}_{\delta}}} }{n}.$$
%However, 
%\noindent We suppose that for all non-zero $\bm{h} \in \mathcal{C}$ we have 
%$$\frac{\bm{h}^{\top}\bm{H}_{\bm{A}_{\delta}}\bm{h}}{||\bm{h}||_2^2} \geq \gamma > 0.$$

%\noindent We need to give more information about $\mathcal{C}$ set.

%\noindent Without loss of generality, we assume that $\lambda_2 = 0$.
\noindent The following theorem establishes an upper bound of the penalized BernSVM with Adaptive Lasso.

 \begin{theorem}
 \label{T2}
 Assume that Assumptions $\bm{A}1-\bm{A}3$ are met and $\lambda_1 \geq  \frac{M(1+\gamma ||\bm{W}_{\bm{S}^{c}}^{-1}||_{\infty})}{\sqrt{n}(\gamma - ||\bm{W}_{\bm{S}}||_{\infty})}$ for any $\gamma > ||\bm{W}_{\bm{S}}||_{\infty}$.
 Then we have $ \hat{\bm{\beta}} - \bm{\beta}^{\star} \in \mathcal{C}$ and 
  the estimator coefficients of the BernSVM with  Adaptive Lasso satisfies  \begin{equation}
\label{bernbond}
||\hat{\bm{\beta}} - \bm{\beta}^{\star}||_2 \leq \frac{(\frac{M}{\sqrt{n}} +  ||\bm{W}_{\bm{S}}||_{\infty} \lambda_1)\sqrt{s}}{\kappa}.
\end{equation}
\end{theorem}
\noindent The proof of Theorem \ref{T2} is given in Appendix \textbf{B}.
In Theorem \ref{T2}, the upper bound depends on $\lambda_1$. Therefore, a choice of $\lambda_1$ is necessary to capture the near-oracle property of the estimator \citep{peng2016error}. 
Note that in Theorem \ref{T2}, we have used the fact that the first derivative of the loss function is bounded. 
In the next lemma, we use a probabilistic upper bound of the first derivative of the BernSVM loss function with aim to give a best choice of $\lambda_1$ and to guarantee a rate of $\mathbf{O}_{P}(\sqrt{s\log(p)/n})$ for the error $||\hat{\bm{\beta}} - \bm{\beta}^{\star}||_2$.

\begin{lemma}
\label{sub3}
\noindent Let $z_j = \frac{1}{n} \sum_{i=1}^{n} - B_{\delta}^{'}(y_i\bm{x}_i^{\top}\bm{\beta}^{\star})y_i x_{ij}$, and  $z^{\star} = \max_{\{j=1,...,p\}}\{|z_j|\}$.

\noindent The variables  $z_j$'s are sub-Gaussian with 
variance $\frac{M^2}{n}$ and for all $t> 0$, $z^{\star}$ satisfies
$$P(z^{\star} > t) \leq 2p\exp(\frac{-t^2n}{2M^2}).$$
\end{lemma}
%An universal choice of the $\lambda_1$ is given in the literature, Peng et al. (2016), by 
%$\lambda_1 = c\sqrt{2A(\eta)\log(p)/n}$, where $c>0$, $\eta$ is a small probability and 
%$A(\eta)>0$.
\noindent The proof of this lemma is given in Appendix {\bf C}. The following theorem outlines the rate of convergence of the error $||\hat{\bm{\beta}} - \bm{\beta}^{\star}||_2$.

\begin{theorem}\label{eq3}
Let $w_{max} = ||\bm{W}_S||_{\infty}$ and 
 $w_{min} = ||\bm{W}^{-1}_{S^{c}}||_{\infty}$ and  $\hat{\bm{\beta}}$ the solution of the 
BernSVM with the adaptive Lasso. For any $\gamma > w_{max}$, we assume that the event $P_1 = 
\{z^{\star} \leq \frac{(\gamma - w_{max})\lambda_1}{1 + \gamma w_{min}}\}$ is satisfied. Let $\lambda_1 = \sqrt{2} M \frac{1 + \gamma w_{min}}{(\gamma - w_{max})} \sqrt{\frac{\log(2p/\xi)}{n}}$ for a small $\xi$. Then, under assumptions $\bm{A}_1-\bm{A}_4$, we have
 $$||\hat{\bm{\beta}} - \bm{\beta}^{\star}||_2 =  \bm{O}_{P}(\sqrt{s\log(p)/n}).$$
\end{theorem}
\noindent The proof of this theorem is similar to the proof of Theorem \ref{T2}. It is detailed in Appendix {\bf C}. 
\begin{rem}
The event $P_1$, defined in Theorem \ref{eq3}, is realized with high probability $1 - \xi$.
This can be verified by taking $t = \sqrt{2}M \sqrt{\frac{\log(2p/\xi)}{n}} $ and applying Lemma \ref{sub3}, whcih leads to
$$P(z^{\star} > \frac{(\gamma - w_{max})\lambda_1}{1 + \gamma w_{min}})= P(z^{\star} > t) \leq 2p\exp(\frac{-2M^2\log(2p/\xi)n}{2M^2n}) = \xi.$$
%\noindent Thus, $$||\hat{\bm{\beta}} - \bm{\beta}^{\star}||_2 = \bm{O}(\frac{c\alpha \sqrt{2sA(\eta)\log(p)/n}}{\kappa} ) = \bm{O}(\sqrt{s\log(p)/n}),$$ wish capture the near-oracle property, Peng et al.(2016) and Didieu (2019)
%of the Lasso Bernstein Svm.
\end{rem}
\noindent In the next section we will derive an upper bound for BernSVM with weighted Lasso, where the weights are estimated.
\subsection{An upper bound of the estimation error for weighted Lasso with estimated weights }
In this section, let $w_j$ be unknown positive weights and $\hat{w}_j$ their corresponding estimates, for $j=1,\ldots,p$, which are assumed to be non-negative. Let $\hat{\bm{\beta}}$ be a minimizer of the weighted Lasso with weights $\hat{w}_j$, defined as follows
$$\hat{\bm{\beta}} = \arg\min_{\bm{\beta}\in \mathbb{R}^p} \frac{1}{n} \sum_{i=1}^{n}B_{\delta}(y_i\bm{x}_i^{\top}\bm{\beta}) + \lambda_1 \sum_{i=1}^{p}\hat{w}_j|\beta_j|.$$
To derive an upper bound of $||\hat{\bm{\beta}} - \bm{\beta}^{\star}||_2$, we suppose that the following event $\Omega_0$ is satisfied  \citep{huang2012estimation}
$$
\Omega_0 = \{\hat{w}_j \leq w_j, \quad \forall j \in \bm{S}\} \cap \{ w_j \leq \hat{w}_j,\quad \forall j \in \bm{S}^{c}\}.
$$
The next theorem gives an upper bound of the BernSVM with weighted Lasso, where the weights are estimated.
\begin{theorem}
Under the same conditions of Theorem 3 and $\lambda_1 = \sqrt{2} M \frac{1 + \gamma w_{min}}{(\gamma - w_{max})} \sqrt{\frac{\log(2p/\xi)}{n}}$, we have 
$$||\hat{\bm{\beta}} - \bm{\beta}^{\star}||_2 =  \bm{O}_{P}(\sqrt{s\log(p)/n}),$$
in the event $\Omega_0\cap P_1$
\end{theorem}
\noindent The proof of this theorem is outlined in Appendix {\bf{D}}.
\begin{rem}
We have shown in Theorem 3 that the event $P_1$ is realized with high probability $1-\xi$. In Theorem 4, we need that the event 
$\Omega_0 \cap P_1$ is realized also with high probability. Indeed,
$P(\Omega_0 \cup P_1) = P(\Omega_0) + P(P_1) - P(\Omega_0 \cap P_1)\leq 1$ leads to $P(\Omega_0 \cap P_1) \geq P(\Omega_0) - \xi$.
\end{rem}

\noindent In the next section, we extend Theorem 3 for non-convex penalties.

\subsection{An extension of the upper bound for non-convex penalties}
We study here the penalized BernSVM estimator with 
the non-convex penalties SCAD and MCP. We establish an upper bound of its error estimation under some regularity conditions. \\
Recall that Algorithm 3 describes steps to compute the solution of BernSVM with the non-convex penalties, as defined in Equation (\ref{weighted1}). The weights $\hat{w}_j, j= 1,\ldots, p,$ are computed iteratively using the first derivative of SCAD or MCP penalty.
To tackle such a problem, we search first for an upper bound of $||\hat{\bm{\beta}}- \bm{\beta}^{\star} ||_2$, where $\hat{\bm{\beta}}$ is an estimator of Equation (\ref{weighted1}).
The next theorem states this result.
\begin{theorem}
\label{T4}
Assume the same conditions of Theorem \ref{eq3} concerning the event $P_1$. Let $\tilde{\bm{\beta}}$ be an initial estimator of $\bm{\beta}$ using Algorithm \ref{al3}, and the weights in (\ref{weighted1}) are given by $\hat{w}_j = P^{'}_{\lambda_1}(|\tilde{\beta_j}|)/\lambda_1$ for $j=1,\ldots,p$.
Then, for any $\gamma > w_{max}$ we have
\begin{equation}
\label{iterate}
||\hat{\bm{\beta}}- \bm{\beta}^{\star} ||_2
\leq \frac{1}{\kappa}\{ \frac{(\gamma - w_{max})\lambda_1}{1 + \gamma w_{min}} + ||P^{'}_{\lambda_1}(|\bm{\beta}^{\star}_S|)||_2 + \frac{1}{a-1}||\tilde{\bm{\beta}}- \bm{\beta}^{\star} ||_2 \}.
\end{equation}
\end{theorem}
\noindent The proof of Theorem \ref{T4} is given in Appendix {\bf{E}}.
We are now able to derive an upper bound of  
$||\hat{\bm{\beta}}^{(l)}- \bm{\beta}^{\star} ||_2$, where $\hat{\bm{\beta}}^{(l)}$ is the estimator of the $l-$th iteration of Algorithm 3.
\begin{theorem}
Assume the same conditions of Theorem 3 and Theorem 4. Let $\tilde{\bm{\beta}}$ be the minimizer of the BernSVM Lasso defined by Equation (\ref{weighted}), where $w_j=1$, for $j=1,\ldots,p$, and $\hat{\bm{\beta}}^{(l)}$ be the estimator of the $l-$th iteration of Algorithm 3. Let $\lambda_1 = \sqrt{2} M \frac{1 + \gamma w_{min}}{(\gamma - w_{max})} \sqrt{\frac{\log(2p/\xi)}{n}}$. Then, we have 
$$||\hat{\bm{\beta}}^{(l)}- \bm{\beta}^{\star} ||_2 =  \bm{O}_{P}(\sqrt{s\log(p)/n}).$$
\end{theorem}

\noindent The proof of Theorem 6 is straightforward and can be outlined as follows. Since $P^{'}_{\lambda_1}(t)$ is concave in $t$, we have $||P^{'}_{\lambda_1}(|\bm{\beta}^{\star}_S|)||_2 \leq \sqrt{s}\lambda_1$, which means that $||P^{'}_{\lambda_1}(|\bm{\beta}^{\star}_S|)||_2 = \bm{O}(\sqrt{s\log(p)/n}).$ Then, by induction using Equation ( \ref{iterate}), for each iteration, we have $||\hat{\bm{\beta}}^{(l)}- \bm{\beta}^{\star} ||_2 =  \bm{O}_{P}(\sqrt{s\log(p)/n}).$
%We have shown in Theorem 3 that the event $P_1$ is realized with high probability $1-\xi$. In Theorem 4 we need that the event 
%$\Omega_0 \cap P_1$ %is realized also with high probability. Indeed,
%$P(\Omega_0 \cup P_1) = P(\Omega_0) + P(P_1) - P(\Omega_0 \cap P_1)\leq 1$ %leads to $P(\Omega_0 \cap P_1) \geq P(\Omega_0) - \xi$.

\section{Simulation and empirical studies}
\label{s4}
We study the empirical behavior of BernSVM with its competitors in terms of computational time, and we evaluate the methods estimation accuracy through different measures of performance via simulations and application to real genetic data sets.   
%In this section, we compare first the timing of the HHSVM in the R package gcdnet and the  penalized BernSVM.
% Secondly, we investigate the performance accuracy of the penalized BernSVM.
\subsection{Simulation study}
In this section, we first compare the computation time of three algorithms: BSVM-GCD, BSVM-IRLS and HHSVM. In the second step, we examine the finite sample performance of eight regularized SVM methods with different penalties in terms of five measures of performance. 
\subsubsection{Simulation study designs}
The data sets are generated following three scenarios described bellow.
%where we set the sample size $n=100$ and the covariates dimension $p=5000$ for the first scenario and $p=1000$ for the second scenario.

\begin{enumerate}
	\item[] {\bf Scenario 1:} In this scenario the columns of $\bm{X} \in \mathbb{R}^{n\times p}$, with $n=100$ and $p=5000$, are generated from a multivariate normal with mean 0 and correlation matrix $\bm{\Sigma}$ with compound symmetry structure, where the correlation between all covariates is fixed to $\rho = 0.5$ in this case. 
	The coefficients are set to be as follows
	$$
	\beta^{\star}_j = (-1)^j\times  \exp(-(2\times j - 1)/20) 
	$$ 
	for $j = 1,2,...,50$ and $\beta^{\star}_j = 0$ for $j = 51,...,p$.
	Firstly, we generated 
 $$
 z_i = \bm{x}_i^{\top}\bm{\beta}^{\star} + \epsilon_i, \quad i=1,\ldots,n,
 $$ 
    where $\epsilon_i \sim N(0,\sigma^2)$.
	The variance $\sigma^2$ is chosen such as a signal to noise ratio ($SNR$) is equal to $3$, where 
 $SNR =\bm{\beta}^{{\star}^{\top}}\bm{\Sigma}\bm{\beta}^{\star}/\sigma^2$. Secondly, the response variable $y_i$ is obtained by passing $z_i$ through an inverse-logistic to obtain the probability $pr_i = P(y_i=1|\bm{x}_i)= 1/(1+exp(-z_i))$. Then $y_i$ is generated from a binomial distribution with probability equal to $pr_i$. By choosing $\delta=2, 1, 0.5 , 0.1, 0.01$, we compare the computation time of the three algorithms mentioned above using lasso penalty ($\lambda_2=0$). The whole process is repeated 100 times and the average run time is reported in Table 1.
\item[] {\bf Scenario 2:} In this scenario, we aimed to study the impact of the correlation magnitude on the run time. As in scenario 1, we generated a data set  
	$(\bm{x}_i,y_i), i=1,\ldots,n$ with $n=100$ and $\bm{x}_i \in \mathbb{R}^{ p}$ with $p=1000$. We consider
	different values for the correlation coefficient $\rho$.  
	The coefficients are generated as follows
	$\beta^{\star}_j \sim U(0.9,1.1)$ for $1\leq j \leq 25$, $\beta^{\star}_j = -1$ for $51\leq j \leq 75$ and $\beta^{\star}_j = 0$ otherwise.
	%following uniform distribution 
	Here, $s = 50$ is the number of significant 
	coefficients. By choosing in this scenario $\delta=0.01, 0.5 , 2$ and $\rho = 0.2,0.5,0.75,0.95$, we compare the computation time of three algorithms using lasso penalty ($\lambda_2=0$). The whole process is repeated 100 times and The average run time is reported in Table 2.
	\item[] {\bf Scenario 3:} This scenario was suggested by \citep{christidis2021data}. Here, it is considered to compare the performance  of sparse classification methods (BernSVM, sparse logistic, HHSVM and SparseSVM) with lasso and elastic net penalties. We generate our data from a logistic model
	$$z_i = \log(\frac{p_i}{1-p_i}) = \beta_0 + \bm{x}^{\top}_{S,i}\bm{\beta}_{S}, \quad i=1....,n,$$ 
	where  $\bm{X} \in \mathbb{R}^{n\times p}$ with $n = 50$ and $p = 800$. $S$ is the set of active variables. The columns of $\bm{X}$ are generated from a multivariate normal with mean $0$ and correlation matrix $\bm{\Sigma}$. We take $\Sigma_{ij} = \rho$ for $i \neq j$ and $\Sigma_{ii} = 1, \quad 1\leq i,j \leq p$, which means that all the variables are equally correlated with $\rho = 0.2,0.5,0.8$, so the correlation matrix is different in each scenario. 
	The number of the active predictors is taken equal to 
	%In  a sparse model, we know that there is a subset of active variables and the rest are inactive. However, for each $p$, we take, 
	$s = [p\xi]$ with $\xi = 0.05, 0.3$. 
	The coefficients of the active variables are generated as $(-1)^vu$ where $v\sim Bernoulli(0.3)$ and $u$ is uniformly distributed on $(0,0.5)$. 
	\noindent The response variable $y_i$, $i=1,\ldots,n$, where $ y_i \in \{-1,1\}$, is generated by passing 
	$z_i$ through an inverse-logistic. The value of intercept $\beta_0$ is fixed such as the  conditional probability $P(y_i=1|\bm{x}_i) = 0.3$.

%	\noindent For each scenario, we simulate a training data with $n=50$ observations and $p=800$ variables to train the model  and test data-set of $200$ observations to test the power of the method by computing different measures of performance described below. 
	
\noindent	In this scenario, for each Monte Carlo replication, we simulate two data sets: 1) a training data of $n = 50$ observations, which is used by the different methods to perform a 10-fold cross-validation in order to choose the best model, for each method; 2) a test data of $n_{test} = 200$ observations, which is used to evaluate the methods performance. The latter is based on five measures, which are given as follows: 
	%The correlation matrix on each scenario is different. 
	%For the correlation matrix, we take : $\Sigma_{i,j} = \rho$ for $i \neq j$ and $\Sigma_{ii} = 1, \quad 1\leq i \leq p$, which means that all the variables are equally correlated with $\rho = 0.2,0.5,0.8$.
	
\begin{itemize}
		\item[•] The misclassification rate, $MR$ is defined by 
		$$MR = \frac{\sum_{i=1}^{n_{test}} [({y_{test}}_{i}- \hat{y}_i)/2]^2}{n_{test}};$$
		\item[•] The sensitivity $SE$ is defined as
		$$SE = 1 - \frac{\sum_{i\in A^{1}} [({y_{test}}_{i}- \hat{y}_i)/2]^2}{n_{test}},$$
		 where $A^1=\{i: {y_{test}}_{i} =1\}$;	
		\item[•] The specificity $SP$ is defined as 
		$$SP = 1 - \frac{\sum_{i\in A^{-1}} [({y_{test}}_i- \hat{y}_i)/2]^2}{n_{test}},$$
		where $A^{-1}=\{i: {y_{test}}_{i} = -1\}$;	
		\item[•]The precision $PR$ is defined by
$$PR = \frac{ \neq \{j: {\beta}^{\star}_j \neq 0, \hat{\beta}_j \neq 0\}}{\neq  \{\hat{\beta}_j \neq 0\}},$$
        \item[•]The recall $RC$ is defined by 
$$RC = \frac{\neq \{j: {\beta}^{\star}_j \neq 0, \hat{\beta}_j \neq 0\}}{\neq \{{\beta}^{\star}_j \neq 0\}},$$
 where $\hat{y}_i$ is the $i-$th predicted for the response variable, $\beta^{\star}_j$ is $j-$th coordinate of the original coefficients and $\hat{\beta}_j$ is the $j-$th coordinate of  the estimated coefficients. We note that the performance of a given method in terms of SE (or SP or PR or RC)  is indicated by its largest value.
	\end{itemize}
\end{enumerate}
\noindent The eight sparse classification methods to be compared in scenario 3 are as follows: 
\begin{itemize}
	\item[.] BernSVM-Lasso: the BSVM-IRLS with Lasso;
	\item[.] BernSVM-EN: the BSVM-IRLS with Elastic Net;
	\item[.] Logistic-Lasso: the binomial model with Lasso computed using glmnet R package \citep{friedman2010regularization};
	\item[.] Logistic-EN: the binomial model with Elastic Net computed using {\it glmnet} R package \citep{friedman2010regularization};
	\item[.] HHSVM$_L$ : HHSVM with Lasso computed using gcdnet R package \citep{yang2013efficient};
	\item[.] HHSVM$_{EN}$: HHSVM with Elastic Net computed using gcdnet R package \citep{yang2013efficient};
	\item[.] SparseSVM-Lasso: the sparse SVM with Lasso using the 
	sparseSVM R package \citep{yi2017semismooth};
	\item[.] SparseSVM-EN: the sparse SVM with Elastic Net using the 
	sparseSVM R package \citep{yi2017semismooth}.
\end{itemize}
Of note, we have used two algorithms to solve the penalized BernSVM optimization problem:  the BSVM-GCD algorithm based on an MM combined with coordinate descent and the BSVM-IRLS algorithm based on 
an IRLS  scheme combined with coordinate descent.
\subsubsection{Simulation results}
%Summing up all the scenarios considered here, we have the following observations:

\begin{table}[H]
\begin{center}
%\begin{center}
\begin{tabular}{ |l|l|l|l|}
  \hline
$\delta$ &  \multicolumn{3}{|c|}{ Time}  \\
    \hline

    & BSVM-GCD &BSVM-IRLS &  HHSVM\\                                                           
    \hline 
0.01&7.50&$\bm{3.62}$&5.84 \\
0.1&2.09&$\bm{1.19}$&1.4\\
0.5&1.55&$\bm{0.53}$&0.86 \\
1&1.60&$\bm{0.61}$&0.77\\
2&2.70&$\bm{0.81}$&0.85\\
\hline

\end{tabular}
\end{center}
\caption{ The run times (in seconds) for  BSVM-GCD, BSVM-IRLS and HHSVM for $\delta = 2 , 1, 0.5 , 0.1, 0.01 $ and $\lambda_2 = 0$ for Scenario 1.}
\end{table}

\begin{figure}[H]
\caption{Comparison of the computation time of the three algorithms as a function of the parameter $\delta$, in Scenario 1}
\begin{center}
\includegraphics[scale=1]{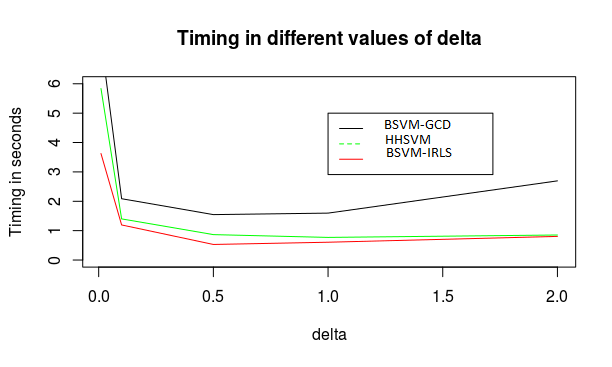} 
\end{center}
\end{figure}  
\noindent Summing up all the scenarios considered here, we have the following observations:

\noindent $\bullet$ Table 1 summarizes the average run time of the HHSVM and the penalized BernSVM with lasso penalty in scenario 1. Figure $1$ provides averaged computation time curve of 100 runs of the three algorithms as a function of $\delta$  values in the same scenario. Our proposed BSVM-IRLS  approach obtains the best computation time followed by 
HHSVM for any $\delta$ value. While BSVM-GCD produces the highest computation time for each delta value. These results are well summarized in Figure 1 by the gaits of the algorithms' computation time curves, where the curve for HHSVM is intermediate between the other two.

$\bullet$ Table 2 summarizes the average run time of the HHSVM and Bernstein penalized SVM with lasso penalty in scenario 2 with different values of $\delta$ and $\rho$. The result is similar to scenario 1 where BSVM-IRLS provides the best run time followed by HHSVM for any value of $\delta$ and $\rho$. It is interesting to notice here that the three algorithms produce their best time for the same $\delta=0.5$ for each value of $\rho$. 

\begin{table}[H]
\begin{center}
%\begin{center}
\begin{tabular}{ |l|l|l|l|l|}
  \hline
 $\rho$&$\delta$ &  \multicolumn{3}{|c|}{Time}  \\
    \hline

  & & BSVM-GCD &BSVM-IRLS &  HHSVM \\
  %\scriptsize{  BernGCD}& \scriptsize{ BernIRLS} &\scriptsize{ HHSVM} &  \scriptsize{BernGCD}&  \scriptsize{BernIRLS} & \scriptsize{HHSVM}\\ 
\hline
&0.01&4.13&$\bm{2.08}$&2.75\\
%0.1&1.325&0.594&0.851&54.117&54.254&54.035&0.422&0.429&0.422  \\
0.2&0.5&1.06&$\bm{0.39}$&0.60 \\
%1&1.067&0.435&0.573&54.086&54.255&54.233&0.411&0.409&0.424 \\
&2&1.86&$\bm{0.55}$ & 0.65 \\
\hline
&0.01&7.70&$\bm{3.52}$&5.50\\
%0.1&2.006&1.158&1.311&53.164&53.342&53.077&0.39&0.426&0.401 \\
0.5&0.5&1.56&$\bm{0.53}$&0.90\\
%1&1.567&0.624&0.839&53.111&53.085&53.56&0.399&0.413&0.412  \\
&2&2.91&$\bm{0.85}$&0.96 \\

\hline
&0.01&11.25&$\bm{4.76}$&7.41 \\
%0.1&3.119&2.719&2.304&54.115&54.236&54.036&0.352&0.357&0.34 \\
0.75&0.5&2.31&$\bm{0.9}$&1.38 \\
%1&2.211&0.909&1.107&54.163&53.996&54.319&0.35&0.341&0.333 \\
&2&4.26&$\bm{1.20}$&1.37 \\

\hline
&0.01&14.69&$\bm{5.84}$&11.31 \\
%0.1&3.738&5.154&2.87&54.932&54.689&55.337&0.286&0.305&0.293 \\
0.9&0.5&2.41&$\bm{1.31}$&1.37\\
%1&2.48&1.35&1.164&55.192&54.817&55.043&0.287&0.307&0.293  \\ 
&2&5.41&$\bm{1.77}$&1.87  \\
\hline

\end{tabular}
\end{center}
\caption{ Timings (in seconds) for the HHSVM and the penalized BernSVM for $\delta = 0.01 , 0.5 , 2 $ and $\lambda_2 = 0$ for different values of $\rho = 0.2,0.5,0.75,0.9$ for Scenario 2} 
\end{table}
   
$\bullet$ Table $3$ report the average of the five performance measures of our methods and the competitor's on Scenario 3, for  $\xi = 0.05,0.3$ and $\rho = 0.2,0.5,0.8$. The results obtained for scenario 3 can be summarized as follows:
\begin{itemize}
	\item[-] For a fixed number of active variables, increasing rho leads to a decrease of about $5-14\%$ in MR, a growth of about $21-40\%$ in $SE$, and a slight increase in $RC$ for methods using the EN penalty(except SparseSVM-EN). While all methods are relatively stable in terms of $SP$ and $PR$ to any increase in $\rho$ and/or in $\xi$.
	\item[-] On the other hand, for any fixed value of $\rho$ and for any increase in $\xi$, $MR$ decreases by about $8-16\%$, $SE$ decreases by about $8-40\%$ and $RC$ increases slightly for only the methods based on the EN penalty and increases greatly of about $11-28\%$ for BernSVM-EN and HHSVM$_{EN}$.
	\item[-] We observe that the methods are relatively comparable in terms of $MR$ with a slight advantage to those using the EN penalty for any pair $(\rho,\xi)$. While for any fixed $\rho$, methods based on standard penalized logistic and SVM provide good performance than the others in terms of $SE$ for any value of $\xi$. The difference with methods based on the smooth approximation of the hinge loss function decreases for increasing the value of $\rho$ and a fixed value of $\xi$.
 In addition, the previous methods are still relatively better in terms of $SP$. While the methods with the lasso penalty (except SparseSVM) and logistic regression with EN penalty are better in terms of $RC$ for any pair $(\rho,\xi)$. Finally, the methods are relatively comparable in terms of $PR$.  
\end{itemize}
Therefore, it can be concluded that our approach produced satisfactory and relatively similar results to HHSVM. 
\subsection{Empirical study}
To illustrate the effectiveness of our BernSVM approach, we also compare classification accuracy, sensitivity and specificity of the latter methods on three large-scale samples of real data sets. 
%In this section, we consider three datasets to compare classification accuracy and the timing of our method and the HHSVM.
The first data set is the DNA methylation measurements studied in \citep{kharoubi2019cluster}, 
and two standard data sets, usually used to evaluate classifiers performance in the literature, namely, the prostate cancer \citep{singh2002gene} and the leukemia  \citep{golub1999molecular} data sets. \\
The DNA methylation measurements are collected around the {\it BLK} gene located in chromosome 8 to detect differentially methylated regions (DMRs, refer to genomic regions with significantly different methylation levels between two groups of samples, e.g. : cases-controls). The dataset contains measurements of DNA methylation levels of 40 different samples of cell types:  B cells (8 samples), T cells (19 samples), and monocytes (13 samples). 
Methylation levels are measured in, $p=5896$, CpG sites (predictors). Since methylation levels are known to be different between B-cell types compared to the T- and Monocyte-cell types around the {\it BLK} gene \citep{kharoubi2019cluster}, we code the cell types as $y=\{-1,1\}$ binary response, with $y=1$ corresponds to B-cells and $y=-1$ corresponds to T- and Monocyte-cell types.
\noindent The second data-set is available in a publicly R package, spls. The prostate data consists of $n = 102$ subjects, 52 patients with prostate tumor and 50 patients used as control subjects.
The data contains expression of $p = 6033$ genes across the subjects.
We code the response variable $y=\{-1,1\}$ binary response, with $y=1$ corresponds to subjects with 
tumor prostate and $y = -1$ corresponds to normal subjects.
\noindent The third one contains $n=72$ observations and $p = 6817$ predictors and comes from a study of gene expression in two types of acute leukemias: acute lymphoblastic leukemia (ALL) and acute myeloid leukemia (AML). Gene expression levels were measured using Affymetrix high density oligonucleotide arrays. Subjects with AML are coded as $y=-1$ class and subjects with ALL as $y=1$ class. In total, we have 25 subjects in class 1 and 47 subjects in class -1. All the datasets were subjects to pre-screening procedure using a two sample t-tests \citep{storey2003statistical} and only $1000$ predictors are retained.

\noindent For the Elastic Net regularization, we take $\lambda_2 = \alpha = 3/4$.
The whole process was repeated 100 times and for each data-set we train the methods using 40$\%$ of the observations to perform a $10$ cross-validation to choose regularization parameters of each method and 60$\%$ of the observations to test the methods and to compute the performance accuracy measures. 
The value of $\delta$ is fixed to 2 for BernSVM and HHSVM.

\noindent We see from Table 4 that on the prostate data, the methods with the EN penalty provide a better result in terms of MR and SP (except SparseSVM). While in terms of SE, all methods provided a comparable result with little advantage to logistic with lasso penalty. For leukemia data, methods with EN penalty (except SparseSVM) dominate for all three performance measures.  For methylation data, the methods with EN penalty produce the best result in terms of MR. While SparseSVM-EN followed by Logistic-EN dominate in terms of SE. More over all methods are relatively comparable in terms of SP, except SparseSVM-Lasso.

\begin{table}[H]
\begin{center}
\begin{tabular}{|c |c|c c c c c|c c c c c|} 
 \hline
& &\multicolumn{5}{|c|}{$\xi = 0.05$ } & \multicolumn{5}{|c|}{$\xi = 0.3$ }\\
 \hline 
$\rho$& Method &MR &SE&SP&RC&PR &MR &SE&SP&RC&PR\\
 \hline
&BernSVM-Lasso&0.32&0.14&0.96&0.04&0.06 &0.22&0.39&0.95&0.04&0.33 \\
&BernSVM-EN&0.30&0.23&0.94&0.44&0.05 &0.163&0.556&0.96&0.576&0.304 \\
&Logistic-Lasso&0.30&0.34&0.89&0.03&0.10 &0.19&0.52&0.93&0.03&0.35 \\
0.2& Logistic-EN&0.29&0.34&0.90&0.05&0.094 &0.17&0.58&0.94&0.05&0.34 \\
&HHSVM$_L$&0.32&0.13&0.95&0.04&0.09 &0.22&0.37&0.96&0.04&0.35  \\
&HHSVM$_{EN}$&0.31&0.18&0.95&0.40&0.06 &0.17&0.52&0.97&0.57&0.30 \\
&SparseSVM-Lasso&0.29&0.34&0.89&0.19&0.06 &0.16&0.61&0.94&0.25&0.30 \\
&SparseSVM-EN&0.29&0.34&0.9&0.20&0.06 &0.16&0.63&0.94&0.26&0.31  \\
\hline
&BernSVM-Lasso&0.25&0.35&0.94&0.03&0.07 &0.12&0.72&0.95&0.04&0.31 \\
&BernSVM-EN&0.23&0.38&0.96&0.43&0.05 &0.08&0.8&0.96&0.54&0.30 \\
&Logistic-Lasso&0.22&0.54&0.90&0.03&0.07 &0.11&0.78&0.94&0.02&0.31  \\
0.5&Logistic-EN&0.21&0.55&0.91&0.04&0.07 &0.09&0.81&0.95&0.05&0.32 \\
&HHSVM$_L$&0.26&0.30&0.95&0.03&0.08 &0.13&0.68&0.95&0.03&0.31 \\
&HHSVM$_{EN}$&0.24&0.34&0.96&0.35&0.06 &0.08&0.79&0.97&0.56&0.31 \\
&SparseSVM-Lasso&0.22&0.52&0.90&0.22&0.06 &0.11&0.75&0.94&0.22&0.32 \\
&SparseSVM-EN&0.23&0.52&0.90&0.23&0.06 &0.11&0.73&0.96&0.22&0.31  \\
\hline
&BernSVM-Lasso&0.19&0.54&0.95&0.02&0.06 &0.08&0.81&0.97&0.03&0.31 \\
&BernSVM-EN&0.17&0.57&0.95&0.43&0.05 &0.05&0.88&0.98&0.55&0.31 \\
&Logistic-Lasso&0.17&0.67&0.91&0.01&0.06 &0.07&0.87&0.96&0.02&0.31 \\
0.8&Logistic-EN&0.16&0.68&0.92&0.03&0.06 &0.05&0.89&0.97&0.05&0.32  \\
&HHSVM$_L$&0.20&0.50&0.95&0.01&0.05 &0.08&0.81&0.97&0.02&0.32  \\
&HHSVM$_{EN}$&0.19&0.50&0.96&0.31&0.06 &0.05&0.88&0.98&0.59&0.31 \\
&SparseSVM-Lasso&0.19&0.60&0.91&0.30&0.06 &0.11&0.78&0.94&0.33&0.32\\
&SparseSVM-EN&0.19&0.61&0.90&0.32&0.06 &0.10&0.78&0.95&0.34&0.31  \\
\hline
\end{tabular}
\end{center}
\caption{Average of the performance measures of our methods and the competitor's on Scenario 3 for  $\xi = 0.05,0.3$ and $\rho = 0.2,0.5,0.8$.}
\end{table}

\begin{table}
\begin{tabular}{ |l|l l l|l l l|l l l |}
  \hline
Data    & \multicolumn{3}{|c|}{ Prostate} & \multicolumn{3}{|c|}{Leukemia} & \multicolumn{3}{|c|}{Methylation} \\
  \hline \hline
  Method   & MR&SE&SP  & MR&SE&SP     &                MR&SE&SP \\ 
    \hline
BernSVM-Lasso&0.08&0.91&0.93& 0.07&0.89&0.94 &0.13&0.28&0.98 \\
BernSVM-EN&0.08&0.9&0.94 & 0.04&0.95&0.96 &0.08&0.68&0.97  \\
Logistic-Lasso&0.09&0.92&0.91 &0.06&0.94&0.94 &0.1&0.57&0.96  \\
Logistic-EN&0.08&0.91&0.92 &0.04&0.95&0.96 &0.07&0.74&0.97    \\
HHSVM$_L$&0.09&0.9&0.92 &0.08&0.89&0.93 &0.12&0.35&0.98   \\
HHSVM$_{EN}$&0.08&0.9&0.94 &0.05&0.95&0.95 &0.09&0.55&0.97   \\
SparseSVM-Lasso&0.12&0.91&0.85 &0.08&0.92&0.92 &0.18&0.39&0.91  \\
SparseSVM-EN&0.12&0.91&0.86 &0.09&0.94&0.9 &0.07&0.78&0.97 \\
\hline
\end{tabular}
%\end{center}
\caption{Average of the three performance measures on three real data sets}
\end{table}
\section{Discussion}
\label{s5}
In this work we aimed to better investigate the binary classification in high dimensional setting using the SVM classifier. We proposed a new smooth loss function, called BernSVM, to approximate the SVM hinge loss function. The BernSVM loss function has nice properties, which allows for efficient implementation of the penalized SVM and helps to derive appropriate theoretical results of the model parameter estimators. We have proposed two algorithms to solve the penalized BernSVM. The first one termed BSVM-GCD and combines the coordinate descent algorithm and the MM principle. The second one, called BSVM-IRLS, and uses an IRLS-type algorithm to solve the underlying optimization problem. We have also derived non asymptotic results of the penalized BernSVM estimator with weighted Lasso, with known weights, and we have extended this result for unknown weights. Furthermore, we have showed that the estimation error of penalized BernSVM with weighted Lasso achieves a rate of order $\sqrt{s\log(p)/n}$. We compared our approach with its competitors through a simulation study and the results showed that our method outperforms its competitors in terms of the computational time while maintaining good performance in terms of prediction and variable selection. The proposed work have shown also accurate results when analyzing three high-dimensional real data sets. 
The penalized BernSVM is implemented in an R package BSVM, which is publicly available for use from Github (lien). \\
%In one hand, the BernSVM loss function can be used with with convex and non-convex penalties. In the other hand, we can used BernSVM loss function with group penalties (such as group Lasso) and the diversity penalty. In our future work we will use this loss function to investigate the ensemble SVM \citep{dietterich2000ensemble}.   
\noindent 
Finally, given the attractive properties of the BernSVM approach to approximate penalized SVM efficiently, one can adapt the proposed method for group penalties, such as group-Lasso/SCAD/MCP. Moreover, recently an interesting penalty has been developed \citep{christidis2021data}, which fits linear regression models that splits the set of covariates into groups using a new diversity penalty. The authors provided also interesting properties of their proposed estimator. Adaptation of the diversity penalty to penalized SVM using the proposed BernSVM approach would be an interesting avenue to investigate. This is left for future work.
  
\section{Appendix}
%%%%%%%%%%%%%%%%%%%%%%%%%%%%%%%%%%%%%%%%%%%
\subsection{Appendix $\bm{A}$ }
\begin{proof} (Theorem 1)
	
\noindent Recall the equivalent problem where
$$
q_{\delta}(x)=g_{\delta}(2\delta x+1-\delta)\,\ x\in\ [0,\ 1],
$$
must satisfy the alternative three Conditions C1', C2' and C3'. Now write
$$
q_{\delta}(x)=\sum_{k=0}^{4}c(k,\ 4;q_{\delta})b_{k,4}(x)\ ,\ x\in\ [0,\ 1],
$$
%\newpage
\noindent and notice that Condition C1' implies, $c(0,4;q_{\delta}) = \delta, c(1,4;q_{\delta}) = \delta/2$ and $c(2,4;q_{\delta}) = 0.$ This will be  seen from the properties of the Bernstein basis above. On the other hand, Condition C2' implies that $c(3,4;q_{\delta})=0=c(4,4;q_{\delta})$. Therefore, we have
$$
q_{\delta}(x)=\delta b_{0,4}(x)+(\delta/2)b_{1,4}(x)=\delta(1-x)^{4}+2\delta x(1-x)^{3},\ x\in\ [0,\ 1].
$$
Notice that Conditions C1' and C2' uniquely determine $q_{\delta}$. It remains to show that it is convex. We have
$$
c(0,2;q_{\delta}) =12\triangle^{2}c(0,4;q_{\delta})=0,\ c(1,2;q_{\delta}) =12\triangle^{2}c(1,4;q_{\delta})=6\delta,
$$
and
$$
c(2,2;q_{\delta}) =12\triangle^{2}c(2,4;q_{\delta})=0.
$$
This shows that $q_{\delta}^{''}(x) > 0$ for all $x \in (0,1)$  and in particular, the Condition C3' is satisfied. 

\noindent In order to compute the derivatives of $g_{\delta}$, we simply use the Bernstein basis properties (again) together with the chain rule as follows
$$ g_{\delta}^{'}(t)=q_{\delta}^{'}((t-1+\delta)/2\delta)/2\delta,$$ 
and 
$$
g_{\delta}^{''}(t)=q_{\delta}^{''}((t-1+\delta)/2\delta)/4\delta^{2},
$$
with
$$
q_{\delta}(x)=\delta b_{0,4}(x)+(\delta/2)b_{1,4}(x)=\delta(1-x)^{4}+2\delta x(1-x)^{3},\ x\in\ [0,\ 1].
$$
Thus, one has 
$$
q_{\delta}^{'}(x)=-2\delta\{b_{0,3}(x)+b_{1,3}(x)\}=-2\delta\{(1-x)^{3}+3x(1-x)^{2}\},
$$
and
$$
q_{\delta}^{''}(x)=6\delta b_{1,2}(x)=12\delta x(1-x). \quad \blacksquare
$$	
\end{proof}
%%%%%%%%%%%%%%%%%%%%%%%%%%%%%%%%%%%%%%%%%%%%%%%%%%%%%
\begin{proof}{(Proposition 3)}

\noindent For the AEN penalty $P_{\lambda_1,\lambda_2}(\beta_j)$, the surrogate function in (\ref{obj}) can be written, for all $j = 1,\ldots,p$, as follows 
$$Q_{\delta}(\beta_j|\beta_0,\tilde{\beta}_j) = \frac{\sum_{i=1}^{n}B_{\delta}(r_i)}{n} + 
\frac{\sum_{i=1}^{n}B_{\delta}^{'}(r_i)y_ix_{ij}}{n}(\beta_j - \tilde{\beta}_j) + \frac{L}{2}(\beta_j - \tilde{\beta}_j)^2 + \lambda_1 w_j |\beta_j|+ \frac{\lambda_2}{2} \beta_j^2.$$ 
Its first derivative is given by 
$$\frac{\partial{Q_{\delta}(\beta_j|\beta_0,\tilde{\beta}_j)}}{\partial{\beta_j}} = \frac{\sum_{i=1}^{n}B_{\delta}^{'}(r_i)y_ix_{ij}}{n} + L (\beta_j-\tilde{\beta}_j ) + \lambda_1 w_j \textit{sign}(\beta_j) + \lambda_2 \beta_j.$$
Then, $\frac{\partial{Q_{\delta}(\beta_j|\beta_0,\tilde{\beta}_j)}}{\partial{\beta_j}} = 0$ implies that 
$\omega\beta_j = Z_j - w_j \textit{sign}(\beta_j)$. Hence, Equation in (\ref{ada}) is obtained using the soft-threshold function $S(a,b)$. Equation (\ref{Elastic}) is a sample case of Equation (\ref{ada}) when the weights $w_j = 1$, for all $j=1\ldots p$.

\end{proof}
%%%%%%%%%%%%%%%%%%%%%%%%%%%%%%%%%%%%%%%%%ùù
\subsection{Appendix $\bm{B}$}
\begin{proof} (Lemma 1) 

\noindent Under assumption $\bm{A} 2$ on $\bm{B}(\bm{x}_0,r_0)$, we have 
$$
\bm{H} = \frac{\bm{X}^{\top}\bm{D}\bm{X}}{n} \succeq  \frac{\kappa_3{\bm{X}^{\top}\bm{X}} }{n},
$$
because $\bm{D} \succeq \kappa_3 \bm{I}$, where $\bm{M}\succeq \bm{N}$ means that $\bm{M} - \bm{N}$ is 
positive semi-definite.
\noindent The true parameter is sparse, then by Cauchy-Schwarz, we have $$||\bm{h}_{S}||_1 = \sum_{j\in \bm{S}} |h_j|. 1 \leq ( \sum_{j\in \bm{S}} |h_j|^2)^{1/2}  (\sum_{j\in \bm{S}} 1)^{1/2} = \sqrt{s} ||\bm{h}_{S}||_2 \leq \sqrt{s} ||\bm{h}||_2.$$
\noindent We have also $\bm{h} \in \mathcal{C}$, so we have $||\bm{W}_{S^{c}}\bm{h}_{S^{c}}||_1\le \gamma||\bm{h}_{S}||_1$. Thus, one has
\begin{eqnarray*}
||\bm{h}||_1 &=&||\bm{h}_{S}||_1 + ||\bm{h}_{S^{c}}||_1 \\
&=& ||\bm{h}_{S}||_1 + ||\bm{W}_{S^{c}}^{-1}\bm{W}_{S^{c}}\bm{h}_{S^{c}}||_1 \\
&\leq & ||\bm{h}_{S}||_1 + ||\bm{W}_{S^{c}}^{-1}||_{\infty} ||\bm{W}_{S^{c}}\bm{h}_{S^{c}}||_1 \\
&\leq &   ||\bm{h}_{S}||_1 + \gamma ||\bm{W}_{S^{c}}^{-1}||_{\infty} ||\bm{h}_{S}||_1 \\
&= & (1 + \gamma ||\bm{W}_{S^{c}}^{-1}||_{\infty} ) ||\bm{h}_{S}||_1 \\
&\leq & \sqrt{s} (1 + \gamma ||\bm{W}_{S^{c}}^{-1}||_{\infty} )||\bm{h}||_2.  \\
\end{eqnarray*}

\noindent Then, one has $$- \frac{\kappa_2\log(p)}{n}||\bm{h}||_1^2 \geq - \frac{\kappa_2  s\log(p)(1 + \gamma ||\bm{W}_{S^{c}}^{-1}||_{\infty} )^2}{n}||\bm{h}||_2^2.$$

\noindent Under Assumption \textbf{A}1, we have also
\begin{eqnarray*}
\frac{\kappa_3||\bm{X}\bm{h}||_2^2}{n} &\geq& \kappa_3 \kappa_1 ||\bm{h}||_2^2 - \kappa_3 \kappa_2 \frac{\log(p)}{n}||\bm{h}||^2_1 \\
&\geq& \kappa_3||\bm{h}||_2^2(\kappa_1 - \frac{\kappa_2 s \log(p)(1 + \gamma ||\bm{W}_{S^{c}}^{-1}||_{\infty} )^2}{n}). \\
\end{eqnarray*} 

\noindent Then, if we take 
$$
n > \frac{2\kappa_2 s \log(p)(1 + \gamma ||\bm{W}_{S^{c}}^{-1}||_{\infty} )^2}{\kappa_1},$$ 
we obtain
$$\kappa_1 - \frac{\kappa_2 s \log(p)(1 + \gamma ||\bm{W}_{S^{c}}^{-1}||_{\infty} )^2}{n} > \kappa_1 - \kappa_1/2 = \kappa_1/2.$$

\noindent Thus, we have 
$$\frac{\kappa_3||\bm{X}\bm{h}||_2^2}{n} \geq \frac{\kappa_3\kappa_1}{2}||\bm{h}||_2^2,$$
which means that the Hessian 
matrix $\bm{H}$ satisfies the RSC with $\kappa = \frac{\kappa_3\kappa_1}{2}$. \quad 
$\blacksquare$
\end{proof}
\begin{proof} (Theorem 2)

\noindent Let $ \bm{h} = \hat{\bm{\beta}} - \bm{\beta}^{\star}$. 
Firstly, we prove that the Lasso BernSVM error satisfies the cone constraint given by 
$$\mathcal{C}(S) = \{\bm{h}\in \mathbb{R}^{p}:||\bm{W}_{\bm{S}^{c}}\bm{h}_{\bm{S}^{c}}||_1 \leq  \gamma||\bm{h}_S||_1 \}.$$

\noindent We have, $L(\hat{\bm{\beta}} ) \leq L(\bm{\beta}^{\star})$ implies that 
$L(\bm{h} + \bm{\beta}^{\star}) \leq L(\bm{\beta}^{\star})$.Thus,

%\noindent Let $\bm{\beta}^{\bm{0}} = \bm{W}\bm{\beta}^{\star}$ and $\bm{h} = \bm{W}\tilde{\bm{h}}$.

$$\frac{1}{n} \sum_{i=1}^{n}B_{\delta}(y_i\bm{x}_i^{\top}(\bm{h} + \bm{\beta}^{\star})) + \lambda_1 ||\bm{W}(\bm{h} + \bm{\beta}^{\star})||_1 \leq \frac{1}{n} \sum_{i=1}^{n}B_{\delta}(y_i\bm{x}_i^{\top} \bm{\beta}^{\star})) + \lambda_1 ||\bm{W}\bm{\beta}^{\star}||_1,$$
which implies that
\begin{equation}
\label{ineq1}
\frac{1}{n} \sum_{i=1}^{n}B_{\delta}(y_i\bm{x}_i^{\top}(\bm{h} + \bm{\beta}^{\star}))- \frac{1}{n} \sum_{i=1}^{n}B_{\delta}(y_i\bm{x}_i^{\top} \bm{\beta}^{\star}))   \leq  \lambda_1( ||\bm{W}\bm{\beta}^{\star}||_1 - ||\bm{W}(\bm{h} + \bm{\beta}^{\star})||_1).
\end{equation}
As the BernSVM loss function is twice differentiable, we can apply a second order Taylor expansion
\begin{equation}
\label{Taylor}
\frac{1}{n} \sum_{i=1}^{n}B_{\delta}(y_i\bm{x}_i^{\top}(\bm{h} + \bm{\beta}^{\star})) - \frac{1}{n} \sum_{i=1}^{n}B_{\delta}(y_i\bm{x}_i^{\top} \bm{\beta}^{\star})) = \frac{1}{n} \sum_{i=1}^{n}B_{\delta}^{'}(y_i\bm{x}_i^{\top}\bm{\beta}^{\star})y_i\bm{x}^{\top}_i\bm{h} + \frac{\bm{h}^{\top}\bm{H}\bm{h}}{2}.
\end{equation}
%\noindent Let $z_j = \frac{1}{n} \sum_{i=1}^{n} - B_{\delta}^{'}(y_i\bm{x}_i^{\top}\bm{\beta}^{\star})y_i x_{ij}$.
%\begin{lemma}[Hoeffding's lemma 1963]
%Let $z$ be a random variable such that $\mathcal{E}[z] = 0$ and $z \in [a,b]$ almost surely. Then 
%for any $t \in \mathbb{R}$, it holds 
%$$\mathbb{E}[e^{tz}] \leq e^{\frac{t^2(b-a)^2}{8}}.$$ 
%In particular, $z \sim subG(\frac{(b-a)^2}{4})$.
%\end{lemma}
%\noindent Hence by applying lemma (5), we conclude that $z_j$ is a sub-Gaussian with 
%variance $\frac{M^2}{n}$.

\noindent We can see that 
\begin{eqnarray*}
\frac{1}{n} \sum_{i=1}^{n} - B_{\delta}^{'}(y_i\bm{x}_i^{\top}\bm{\beta}^{\star})y_i \bm{x}^{\top}_i\bm{h}&\leq& \frac{1}{n} \sum_{i=1}^{n}|B_{\delta}^{'}(y_i\bm{x}_i^{\top}\bm{\beta}^{\star})y_i \bm{x}^{\top}_i\bm{h}|\\
& \leq & \frac{1}{n} \sum_{i=1}^{n} |\bm{x}^{\top}_i\bm{h}| \quad \text{because} \quad |B_{\delta}^{'}(y_i\bm{x}_i^{\top}\bm{\beta}^{\star})y_i| = |B_{\delta}^{'}(y_i\bm{x}_i^{\top}\bm{\beta}^{\star})| \leq 1 \\
&= & \frac{1}{n} \sum_{i=1}^{n} |\sum_{j=1}^{p} x_{ij}h_j |\\
&\leq & \sum_{j=1}^{p} (\frac{1}{n} \sum_{i=1}^{n} |x_{ij}|)|h_j|\\
&\leq &  \sum_{j=1}^{p} (\frac{1}{\sqrt{n}} \sqrt{\sum_{i=1}^{n} |x_{ij}|^2})|h_j| \\
&\leq & \frac{M}{\sqrt{n}} ||\bm{h}||_1,  \quad \text{because} \quad ||\bm{x}_{j}||_2 \leq M, \quad \forall  j.\\
\end{eqnarray*}
%\noindent The third inequality by applying Cauchy-Schwartz and the forth inequality by the first assumption. 

\noindent Then, (\ref{ineq1}) and (\ref{Taylor}) give
\begin{eqnarray*}
\frac{\bm{h}^{\top}\bm{H}\bm{h}}{2} &\leq& \frac{1}{n} \sum_{i=1}^{n} - B_{\delta}^{'}(y_i\bm{x}_i^{\top}\bm{\beta}^{\star})y_i\bm{x}^{\top}_i\bm{h} + \lambda_1(||\bm{W}\bm{\beta}^{\star}||_1 - ||\bm{W}(\bm{h} + \bm{\beta}^{\star})||_1),\\
%&=& \frac{1}{n}\sum_{j=1}^{p} \sum_{i=1}^{n} - B_{\delta}^{'}(y_i\bm{x}_i^{\top}\bm{\beta}^{\star})y_i x_{ij} h_j + \lambda_1( ||\hat{\bm{W}}\bm{\beta}^{\star}||_1 - ||\hat{\bm{W}}(\bm{h} + \bm{\beta}^{\star})||_1),\\
%&=& \sum_{j=1}^{p} z_j h_j + \lambda_1( ||\hat{\bm{W}}\bm{\beta}^{\star}||_1 - ||\hat{\bm{W}}(\bm{h} + \bm{\beta}^{\star})||_1),\\
%&\leq& \sum_{j=1}^{p} |z_j|  |h_j| + \lambda_1( ||\hat{\bm{W}}\bm{\beta}^{\star}||_1 - ||\hat{\bm{W}}(\bm{h} + \bm{\beta}^{\star})||_1),\\
%&\leq& \sum_{j=1}^{p} \max_{\{j=1,...,p\}}\{|z_j|\}  |h_j| + \lambda_1(||\bm{\beta}^{\star}||_1 - ||\bm{\beta}^{\star} + \bm{h}||_1 ),\\
%&=&  \max_{\{j=1,...,p\}}\{|z_j|\} \sum_{j=1}^{p}  |h_j| +\lambda_1( ||\hat{\bm{W}}\bm{\beta}^{\star}||_1 - ||\hat{\bm{W}}(\bm{h} + \bm{\beta}^{\star})||_1),\\
&\leq&  \frac{M}{\sqrt{n}} ||\bm{h}||_1 + \lambda_1(||\bm{W}\bm{\beta}^{\star}||_1 - ||\bm{W}(\bm{h} + \bm{\beta}^{\star})||_1).\\
\end{eqnarray*}
%Let $\bm{\beta}^{\bm{0}} = \bm{W}\bm{\beta}^{\star} $ and $\tilde{\bm{h}} = \bm{W}\bm{h}$.

%\noindent We have by applying the triangular inequality  $||\bm{\beta}^{\bm{0}}||_1 - ||\bm{\beta}^{\bm{0}} + \tilde{\bm{h}}||_1 = ||\bm{\beta}^{\bm{0}}_S||_1 - ||\bm{\beta}^{\bm{0}}_S + \tilde{\bm{h}}_S||_1 - ||\tilde{\bm{h}}_{S^{c}}||_1 \leq || \tilde{\bm{h}}_S||_1 - ||\tilde{\bm{h}}_{S^{c}}||_1$.

\noindent Moreover, we have 
\begin{eqnarray*}
||\bm{W}\bm{\beta}^{\star}||_1 - ||\bm{W}(\bm{h} + \bm{\beta}^{\star})||_1 &=& 
||\bm{W}_{\bm{S}}\bm{\beta}^{\star}_{\bm{S}}||_1 - ||\bm{W}_{\bm{S}}(\bm{h}_{\bm{S}} + \bm{\beta}^{\star}_{\bm{S}})||_1 - ||\bm{W}_{\bm{S}^{c}}\bm{h}_{\bm{S}^{c}}||_1, \\
&\leq & ||\bm{W}_{\bm{S}}\bm{h}_{\bm{S}}||_1 - ||\bm{W}_{\bm{S}^{c}}\bm{h}_{\bm{S}^{c}}||_1, \\
\end{eqnarray*}
where we have used in the first equality  $\bm{\beta}^{\star}_{\bm{S}^{c}} = \bm{0}$, and the second inequality is based on the fact that  
$$||\bm{W}_{\bm{S}}\bm{\beta}^{\star}_{\bm{S}}||_1 - ||\bm{W}_{\bm{S}}(\bm{h}_{\bm{S}} + \bm{\beta}^{\star}_{\bm{S}})||_1 \leq |||\bm{W}_{\bm{S}}\bm{\beta}^{\star}_{\bm{S}}||_1 - ||\bm{W}_{\bm{S}}(\bm{h}_{\bm{S}} + \bm{\beta}^{\star}_{\bm{S}})||_1| \leq ||\bm{W}_{\bm{S}}\bm{h}_{\bm{S}}||_1.$$
\noindent Thus,  we obtain %in the event $\bm{\Omega}_0$,
\begin{eqnarray*}
 0 &< &  \frac{\bm{h}^{\top}\bm{H}\bm{h}}{2}  \\
 &\leq& \frac{M}{\sqrt{n}} ||\bm{h}||_1 +\lambda_1 (||\bm{W}_{\bm{S}}\bm{h}_{\bm{S}}||_1 - ||\bm{W}_{\bm{S}^{c}}\bm{h}_{\bm{S}^{c}}||_1).\\
\end{eqnarray*}

%\noindent However, the lemma 3 implies that for some $t>0$ we have $\max_{\{j=1,...,p\}}\{|z_j|\} \leq t$ with probability $1-2p\exp(-t^2/2\sigma^2)$, where $\sigma^2  = \frac{M^2}{n}$. 

\noindent Because the right hand of inequality is positive, we have, with probability $1-\exp(-c_0 n)$, that 
\begin{eqnarray*}
0 &\leq& \frac{M}{\sqrt{n}}||\bm{h}||_1 + \lambda_1(||\bm{W}_{\bm{S}}\bm{h}_{\bm{S}}||_1 - ||\bm{W}_{\bm{S}^{c}}\bm{h}_{\bm{S}^{c}}||_1)\\
& = & \frac{M}{\sqrt{n}}||\bm{h}_{\bm{S}}||_1 + \frac{M}{\sqrt{n}}||\bm{h}_{\bm{S}^{c}}||_1 + \lambda_1(||\bm{W}_{\bm{S}}\bm{h}_{\bm{S}}||_1 - ||\bm{W}_{\bm{S}^{c}}\bm{h}_{\bm{S}^{c}}||_1)\\
&=&  \frac{M}{\sqrt{n}}||\bm{h}_{\bm{S}}||_1 + \frac{M}{\sqrt{n}}||\bm{W}_{\bm{S}^{c}}^{-1}\bm{W}_{\bm{S}^{c}}\bm{h}_{\bm{S}^{c}}||_1 + \lambda_1(||\bm{W}_{\bm{S}}\bm{h}_{\bm{S}}||_1 - ||\bm{W}_{\bm{S}^{c}}\bm{h}_{\bm{S}^{c}}||_1)\\
&\leq &  \frac{M}{\sqrt{n}}||\bm{h}_{\bm{S}}||_1 +  \frac{M}{\sqrt{n}}||\bm{W}_{\bm{S}^{c}}^{-1}||_{\infty}||\bm{W}_{\bm{S}^{c}}\bm{h}_{\bm{S}^{c}}||_1 + \lambda_1||\bm{W}_{\bm{S}}||_{\infty}||\bm{h}_{\bm{S}}||_1 - \lambda_1||\bm{W}_{\bm{S}^{c}}\bm{h}_{\bm{S}^{c}}||_1\\
& = & (\frac{M}{\sqrt{n}} + \lambda_1 ||\bm{W}_{\bm{S}}||_{\infty})||\bm{h}_{\bm{S}}||_1 - (\lambda_1 -  \frac{M}{\sqrt{n}}||\bm{W}_{\bm{S}^{c}}^{-1}||_{\infty})||\bm{W}_{\bm{S}^{c}}\bm{h}_{\bm{S}^{c}}||_1.\\
\end{eqnarray*}
This implies that the Lasso BernSVM error satisfies the cone constraint given by 
$$\mathcal{C}(S) = \{\bm{h}\in \mathbb{R}^{p}:||\bm{W}_{\bm{S}^{c}}\bm{h}_{\bm{S}^{c}}||_1 \leq  \gamma||\bm{h}_S||_1 \},$$
because $\lambda_1 \geq  \frac{M(1+\gamma ||\bm{W}_{\bm{S}^{c}}^{-1}||_{\infty})}{\sqrt{n}(\gamma - ||\bm{W}_{\bm{S}}||_{\infty})}$ for any $\gamma > ||\bm{W}_{\bm{S}}||_{\infty}$.
%\noindent However, we need another assumption to have $\lambda_1 >  ||\bm{W}_{\bm{S}^{c}}^{-1}||_{\infty}$. 

\noindent Thus, we have 
 that the error $\bm{h}$ belongs to the set $\mathcal{C}(S) $. 

%\noindent \noindent Let $\alpha = \max\{||\bm{W}_{\bm{S}}||_{\infty},||\bm{W}_{\bm{S}^{c}}^{-1}||_{\infty}^{-1})\}$.
\noindent Then, we derive the upper bound of the error $\bm{h}$.
From the inequality 
above  and the definition of RSC
%Assumption \textbf{A}2 
we have 
\begin{eqnarray*}
\kappa ||\bm{h}||_2^2 &\leq& (\frac{M}{\sqrt{n}} + \lambda_1 ||\bm{W}_{\bm{S}}||_{\infty}) ||\bm{h}_S||_1  -  (\lambda_1 -  \frac{M}{\sqrt{n}}||\bm{W}_{\bm{S}^{c}}^{-1}||_{\infty}||) ||\bm{h}_{S^{c}}||_1\\
&\leq& (\frac{M}{\sqrt{n}} + \lambda_1 ||\bm{W}_{\bm{S}}||_{\infty}) ||\bm{h}_S||_1, \\
%&\leq& (1 + \lambda_1 \alpha) ||\bm{h}_S||_1\\
&\leq & (\frac{M}{\sqrt{n}} + \lambda_1 ||\bm{W}_{\bm{S}}||_{\infty})\sqrt{s}||\bm{h}||_2.\\
\end{eqnarray*}
The second inequality comes from the fact that $(\lambda_1 -  \frac{M}{\sqrt{n}}||\bm{W}_{\bm{S}^{c}}^{-1}||_{\infty}||)> 0$. Thus, we obtain an upper bound of the error $\bm{h} = \hat{\bm{\beta}} - \bm{\beta}^{\star}$ given by 

$$||\bm{h}||_2 \leq \frac{(\frac{M}{\sqrt{n}} +  ||\bm{W}_{\bm{S}}||_{\infty} \lambda_1)\sqrt{s}}{\kappa}. \quad \blacksquare $$
\end{proof}
\begin{proof}(Theorem 3)
\noindent In Theorem 2, we used $\textbf{A2}$ and the fact that the BernSVM first derivative is bounded by 1 in order to upper bound the random variables $z_j = \frac{1}{n} \sum_{i=1}^{n} - B_{\delta}^{'}(y_i\bm{x}_i^{\top}\bm{\beta}^{\star})y_i x_{ij}$ for $j=1,\ldots,p$. Lemma 2 shows that the variables  $z_j$ are sub-Gaussian's. Thus, we can see that 
$$\frac{1}{n} \sum_{i=1}^{n} - B_{\delta}^{'}(y_i\bm{x}_i^{\top}\bm{\beta}^{\star})y_i \bm{x}^{\top}_i\bm{h} \leq  z^{\star}||\bm{h}||_1,$$
where $z^{\star}=\max_{j} z_j$.
Following the same steps to prove Theorem 2, we obtain 
$$||\bm{h}||_2 \leq \frac{(z^{\star} +  ||\bm{W}_{\bm{S}}||_{\infty} \lambda_1)\sqrt{s}}{\kappa}.$$
We have also that the event $P_1$ is satisfied with high probability, then we have 
$$||\bm{h}||_2 \leq \frac{(z^{\star} +  w_{\max} \lambda_1)\sqrt{s}}{\kappa} \leq \frac{\gamma(1+w_{\min}w_{\max})\lambda_1\sqrt{s}}{\kappa}.$$
We also have that $\lambda_1 = \bm{O}(\sqrt{\log(p)/n})$. Finally, we obtain $$||\bm{h}||_2 = ||\hat{\bm{\beta}} - \bm{\beta}^{\star}||_2 =  \bm{O}_{P}(\sqrt{s\log(p)/n}). \quad \blacksquare$$
\end{proof}
\subsection{Appendix $\bm{C}$}
\begin{proof} (Lemma 2)

\noindent To prove Lemma 2, we provide some results about the sub-Gaussian random variables
\begin{lemma}
\label{sub2}
$z_1,z_2,...,z_p$ are $p$ real random variables with $z_j \sim subG(\sigma^2)$ not necessarily independent, then for all $t>0$ $$\mathbb{P}(max_{j=1,...,p}|z_j|>t)\leq 2p\exp(-t^2/2\sigma^2).$$
\end{lemma}

\begin{lemma}[Hoeffding's Lemma 1963]
\label{sub1}
Let $z$ be a random variable such that $\mathcal{E}[z] = 0$ and $z \in [a,b]$ almost surely. Then 
for any $t \in \mathbb{R}$, it holds 
$$\mathbb{E}[e^{tz}] \leq e^{\frac{t^2(b-a)^2}{8}}.$$ 
In particular, $z \sim subG(\frac{(b-a)^2}{4})$.
\end{lemma}

\noindent Let $z_j = \frac{1}{n} \sum_{i=1}^{n} - B_{\delta}^{'}(y_i\bm{x}_i^{\top}\bm{\beta}^{\star})y_i x_{ij}$.

\noindent Hence, by applying Lemma 3, we conclude that $z_j$ is a sub-Gaussian with 
variance $\frac{M^2}{n}$.

\noindent We can see that 
\begin{eqnarray*}
|z_j| = |\frac{1}{n} \sum_{i=1}^{n} - B_{\delta}^{'}(y_i\bm{x}_i^{\top}\bm{\beta}^{\star})y_i x_{ij}|&\leq& \frac{1}{n} \sum_{i=1}^{n}|B_{\delta}^{'}(y_i\bm{x}_i^{\top}\bm{\beta}^{\star})||y_i| |x_{ij}|\\
& \leq & \frac{1}{n} \sum_{i=1}^{n} |x_{ij}| \\
&\leq & \frac{1}{\sqrt{n}} \sqrt{\sum_{i=1}^{n} x_{ij}^2} \\
&\leq & \frac{M}{\sqrt{n}}.
\end{eqnarray*}
\noindent The third inequality by applying Cauchy-Schwartz and the fourth inequality by Assumption $\bm{A}3$.

\noindent We conclude that the variables $z_j$ are bounded. Moreover, using the fact that $\bm{\beta}^{\star}$ minimize the population version of the loss function and Assumption $\bm{A} 4$, we have $\mathbb{E}[z_j] = 0$, then, the $z_j$ are sub-Gaussian. $\quad \blacksquare$
\end{proof}
\subsection{Appendix $\bm{D}$}
\begin{proof} (Theorem 4)

\noindent We can see that 

$$ \frac{1}{n} \sum_{i=1}^{n} - B_{\delta}^{'}(y_i\bm{x}_i^{\top}\bm{\beta}^{\star})y_i \bm{x}^{\top}_i\bm{h} \leq  z^{\star}||\bm{h}||_1,$$
where $z^{\star}=\max_{j}\frac{1}{n} \sum_{i=1}^{n} - B_{\delta}^{'}(y_i\bm{x}_i^{\top}\bm{\beta}^{\star})y_i x_{ij}$.

\noindent Then, following the same argument as in the proof of Theorem 2, we have
\begin{eqnarray*}
 0 &< &  \frac{\bm{h}^{\top}\bm{H}\bm{h}}{2}  \\
 &\leq& z^{\star} ||\bm{h}||_1 +\lambda_1 (||\hat{\bm{W}}_{\bm{S}}\bm{h}_{\bm{S}}||_1 - ||\hat{\bm{W}}_{\bm{S}^{c}}\bm{h}_{\bm{S}^{c}}||_1)\\
& \leq & z^{\star}||\bm{h}||_1 +\lambda_1 (||\bm{W}_{\bm{S}}\bm{h}_{\bm{S}}||_1 - ||\bm{W}_{\bm{S}^{c}}\bm{h}_{\bm{S}^{c}}||_1).\\
\end{eqnarray*}
%\noindent However, the lemma 3 implies that for some $t>0$ we have $\max_{\{j=1,...,p\}}\{|z_j|\} \leq t$ with probability $1-2p\exp(-t^2/2\sigma^2)$, where $\sigma^2  = \frac{M^2}{n}$. 

\noindent In the fact that the right hand of inequality is positive, and the second inequality is because the event $\Omega_0$ is realized.
 Thus, we have with probability $1-\exp(-c_0 n)$ that 

\begin{eqnarray*}
0 &\leq& z^{\star} ||\bm{h}||_1 + \lambda_1(||\bm{W}_{\bm{S}}\bm{h}_{\bm{S}}||_1 - ||\bm{W}_{\bm{S}^{c}}\bm{h}_{\bm{S}^{c}}||_1),\\
& = & z^{\star} ||\bm{h}_{\bm{S}}||_1 + z^{\star} ||\bm{h}_{\bm{S}^{c}}||_1 + \lambda_1(||\bm{W}_{\bm{S}}\bm{h}_{\bm{S}}||_1 - ||\bm{W}_{\bm{S}^{c}}\bm{h}_{\bm{S}^{c}}||_1)\\
&=& z^{\star} ||\bm{h}_{\bm{S}}||_1 + z^{\star} ||\bm{W}_{\bm{S}^{c}}^{-1}\bm{W}_{\bm{S}^{c}}\bm{h}_{\bm{S}^{c}}||_1 + \lambda_1(||\bm{W}_{\bm{S}}\bm{h}_{\bm{S}}||_1 - ||\bm{W}_{\bm{S}^{c}}\bm{h}_{\bm{S}^{c}}||_1)\\
&\leq & z^{\star} ||\bm{h}_{\bm{S}}||_1 + z^{\star} w_{min}||\bm{W}_{\bm{S}^{c}}\bm{h}_{\bm{S}^{c}}||_1 + \lambda_1 w_{max}||\bm{h}_{\bm{S}}||_1 - \lambda_1||\bm{W}_{\bm{S}^{c}}\bm{h}_{\bm{S}^{c}}||_1)\\
& = & (z^{\star} + \lambda_1 w_{max})||\bm{h}_{\bm{S}}||_1 - (\lambda_1 - z^{\star} w_{min})||\bm{W}_{\bm{S}^{c}}\bm{h}_{\bm{S}^{c}}||_1.\\
\end{eqnarray*}

\noindent which implies that the Lasso BernSVM error satisfies the cone constraint given by 
$$\mathcal{C}(S) = \{\bm{h}\in \mathbb{R}^{p}:||\bm{W}_{\bm{S}^{c}}\bm{h}_{\bm{S}^{c}}||_1 \leq  \gamma||\bm{h}_S||_1 \},$$
because the event $P_1$ is realized with high probability, which means that  
\begin{equation*}
%\label{gama}
\gamma > \frac{(z^{\star} + \lambda_1 w_{max})}{(\lambda_1 - z^{\star}w_{min})}.
\end{equation*}
%\noindent However, we need another assumption to have $\lambda_1 >  ||\bm{W}_{\bm{S}^{c}}^{-1}||_{\infty}$. 

\noindent Thus, we have 
 that the error $\bm{h}$ belongs to the set $\mathcal{C}(S) $. 

%\noindent \noindent Let $\alpha = \max\{||\bm{W}_{\bm{S}}||_{\infty},||\bm{W}_{\bm{S}^{c}}^{-1}||_{\infty}^{-1})\}$.

\noindent  Let $\alpha = \max\{w_{max},w_{min}^{-1})\}$.
Moreover, from the inequality 
above  and  the RSC definition,
%Assumption $(\bm{A} 2)$, 
we have 
\begin{eqnarray*}
\kappa ||\bm{h}||_2^2 &\leq& (z^{\star} + \lambda_1 w_{max})||\bm{h}_{\bm{S}}||_1 - (\lambda_1 - z^{\star} w_{min})||\bm{W}_{\bm{S}^{c}}\bm{h}_{\bm{S}^{c}}||_1\\
&\leq& (w_{max} + w_{min}^{-1}) \lambda_1 ||\bm{h}_S||_1 , \\
&\leq & 2 \alpha \lambda_1\sqrt{s}||\bm{h}||_2.\\
\end{eqnarray*}
\noindent Then, we obtain that 
	$$||\bm{h}||_2 \leq  \frac{2 \alpha \lambda_1\sqrt{s}}{\kappa}.$$
\noindent Thus, 
$$||\hat{\bm{\beta}} - \bm{\beta}^{\star}||_2 =  \bm{O}_{P}(\sqrt{s\log(p)/n}). \quad \blacksquare$$
\end{proof}
\subsection{Appendix $\bm{E}$}
\begin{proof} (Theorem 5)

\noindent We have 
\begin{eqnarray*}
 \kappa||\bm{h}||^2_2 &< &  \frac{\bm{h}^{\top}\bm{H}\bm{h}}{2}  \\
 &\leq& z^{\star} ||\bm{h}||_2 +\lambda_1 (||\hat{\bm{W}}_{\bm{S}}\bm{h}_{\bm{S}}||_1 - ||\hat{\bm{W}}_{\bm{S}^{c}}\bm{h}_{\bm{S}^{c}}||_1)\\
&\leq& z^{\star} ||\bm{h}||_2 +\lambda_1 ||\hat{\bm{W}}_{\bm{S}}\bm{h}_{\bm{S}}||_1.\star\\
\end{eqnarray*}
For all $j \in S$, we have $\lambda_1\hat{w}_j = P^{'}_{\lambda_1}(|\tilde{\beta_j}|)$.
We have also from Taylor expansion that
$$P^{'}_{\lambda_1}(|\tilde{\beta_j}|) = P^{'}_{\lambda_1}(|\beta^{\star}_j|) + P^{''}_{\lambda_1}(|\beta^{\star}_j|)(\tilde{\beta}_j-\beta^{\star}_j)\leq P^{'}_{\lambda_1}(|\beta^{\star}_j|) + \frac{1}{a-1}(\tilde{\beta}_j-\beta^{\star}_j).$$
Thus from ($\star$), we obtain 
\begin{eqnarray*}
 \kappa||\bm{h}||^2_2 &< &  ||\bm{h}||_2(z^{\star} + ||P^{'}_{\lambda_1}(|\bm{\beta}^{\star}_{S}|)||_2 + \frac{1}{a-1} ||\tilde{\bm{\beta}}-\bm{\beta}^{\star}||_2). 
\end{eqnarray*}
Therefore,
$$||\bm{h}||_2 \leq \frac{1}{\kappa}(\frac{(\gamma - w_{max})\lambda_1}{1 + \gamma w_{min}} + ||P^{'}_{\lambda_1}(|\bm{\beta}^{\star}_{S}|)||_2 + \frac{1}{a-1} ||\tilde{\bm{\beta}}-\bm{\beta}^{\star}||_2). \quad \blacksquare$$
\end{proof}
\newpage
%\printbibliography
\bibliographystyle{apalike}
\bibliography{biblio}
\end{document}